\documentclass[twoside]{article}

\usepackage[accepted]{aistats2025}

\usepackage[round]{natbib}

\usepackage{hyperref}       
\usepackage{url}            
\usepackage{booktabs}       
\usepackage{amsfonts}       
\usepackage{nicefrac}       
\usepackage{microtype}      
\usepackage{xcolor}         

\usepackage{algorithm}
\usepackage{algorithmic}
\usepackage{authblk}
\usepackage{amsmath}
\usepackage{amssymb}
\usepackage{mathtools}
\usepackage{amsthm}
\usepackage{wrapfig}

\usepackage{caption}
\usepackage{subcaption}

\theoremstyle{plain}
\newtheorem{theorem}{Theorem}[section]
\newtheorem{proposition}[theorem]{Proposition}
\newtheorem{lemma}[theorem]{Lemma}
\newtheorem{corollary}[theorem]{Corollary}
\theoremstyle{definition}

\newtheorem{assumption}[theorem]{Assumption}
\theoremstyle{remark}

\begin{document}

%

%

\twocolumn[

\aistatstitle{High Dimensional Bayesian Optimization using Lasso Variable Selection}

\aistatsauthor{ Vu Viet Hoang$^1$ \textsuperscript{*} \And Hung The Tran$^2$ \textsuperscript{*} \And  Sunil Gupta$^3$ \And Vu Nguyen$^4$ }
\aistatsaddress{
$^1$FPT Software AI Center\\
$^2$VNPT Media, Vietnam\\
$^3$Applied AI Institute, Deakin University\\
$^4$Amazon, Australia}
]


\let\thefootnote\relax\footnotetext{*These authors contributed equally to this work.}
\begin{abstract}
Bayesian optimization (BO) is a leading method for optimizing expensive black-box optimization and has been successfully applied across various scenarios. However, BO suffers from the curse of dimensionality, making it challenging to scale to high-dimensional problems. Existing work has adopted a variable selection strategy to select and optimize only a subset of variables iteratively. Although this approach can mitigate the high-dimensional challenge in BO, it still leads to sample inefficiency. To address this issue, we introduce a novel method that identifies important variables by estimating the length scales of Gaussian process kernels. Next, we construct an effective search region consisting of multiple subspaces and optimize the acquisition function within this region, focusing on only the important variables. We demonstrate that our proposed method achieves cumulative regret with a sublinear growth rate in the worst case while maintaining computational efficiency. Experiments on high-dimensional synthetic functions and real-world problems show that our method achieves state-of-the-art performance. 
\end{abstract}

\section{Introduction}
\label{sec:intro}
Black-box optimization is crucial in various fields of science, engineering, and beyond. It is widely used for optimizing tasks such as calibrating intricate experimental systems and adjusting hyperparameters of machine learning models. The demand for optimization methods that are both scalable and efficient is widespread. Bayesian optimization (BO) \citep{binois2022survey,garnett2023bayesian,husain2023distributionally} is an algorithm known for its sample efficiency in solving such problems. It works by iteratively fitting a surrogate model, usually a Gaussian process (GP), and maximizing an acquisition function to determine the next evaluation point. Bayesian optimization algorithms \cite{ament2023unexpected} have proven particularly successful in a wide variety of domains including hyperparameter tuning \citep{bergstra2011algorithms,klein2017fast,hvarfner2022pi}, reinforcement learning \citep{marco2017virtual,parker2022automated}, neural architecture search \citep{kandasamy2015high,nguyen2021optimal}, robotics \citep{lizotte2007automatic}.

Although successfully applied in many different fields, Bayesian optimization is often limited to low-dimensional problems, typically at most twenty dimensions \citep{snoek2012practical,frazier2018tutorial,nguyen2020knowing}. Scaling BO to high-dimensional problems has received a lot of interest. The key challenge is that the computational cost of the BO increases exponentially with the number of dimensions, making it difficult to find an optimal solution in a reasonable amount of time. As a result, developing more efficient and scalable BO algorithms for high-dimensional problems is an active area of research \citep{rana2017high}, \cite{letham2020re}, \cite{NEURIPS2020_bb073f28}, \cite{wan2021think}. The decomposition-based methods \cite{han2021high, hoang2018decentralized, kandasamy2015high, mutny2018efficient, rolland2018high,lu2022additive}, assume that the high-dimensional function to be optimized has a certain structure, commonly characterized by additivity. By decomposing the original high-dimensional function into the sum of several low-dimensional functions, they optimize each low-dimensional function to obtain the point in the high-dimensional space. However, it is not easy to decide whether a decomposition exists to learn the decomposition. Other methods often assume that the original high-dimensional function with dimension D has a low-dimensional subspace with dimension $d \ll D$, and then perform the optimization in the low-dimensional subspace and project the low-dimensional point back for evaluation. For example, embedding-based methods \citep{letham2020re, nayebi2019framework, wang2016bayesian,Nguyen_Tran} use a random matrix to embed the original space into the low-dimensional subspace.
Recently, several methods \citep{xu2024standard,hvarfner2024vanilla} have been proposed to place priors on the hyperparameters of GP models. SAASBO \citep{eriksson2021high} uses sparsity-inducing before performing variable selection implicitly, making the coefficients of unimportant variables near zero and thus restraining over-exploration on these variables. Moreover, SAASBO still optimizes the high-dimensional acquisition function, and also due to its high computational cost of inference, it is very time consuming \citep{shen2021computationally, santoni2024comparison}.

Another approach involves directly selecting a subset of variables, which can help avoid the time-consuming matrix operations required by embedding-based methods. In \cite{li2018high}, $d$ variables are randomly selected in each iteration. It is important to note that for both the embedding and variable selection methods, the parameter $d$ can significantly affect performance, but it can be challenging to determine the ideal value for real-world applications. Recently, \cite{song2022monte} propose a variable selection method MCTS\_VS based on Monte Carlo tree search (MCTS), to automatically select a subset of important variables and perform BO only on these variables. However, MCTS\_VS does not demonstrate how changes in variables impact the function's value. In our experiment, we designed the functions with a few important variables. We see that MCTS\_VS is insufficient for identifying these variables. \cite{song2022monte} also provided a regret analysis for their method, however, the regret bound grows linearly in the number of samples, which is an unexpected property of BO.


In this paper, we propose a novel variable selection method for high-dimensional BO. Our main idea is to iteratively categorize all variables into important and unimportant groups and then construct an effective search space where the optimization at the acquisition step is performed only on important variables. Our main contributions are:
\begin{itemize}
    \item  A provably efficient method to upper bound the derivatives in terms of inverse length scales of GP kernels. To our knowledge, this result has not been provided in the literature, despite its simplicity.
    
    \item A design for an effective subspace in high dimensional BO where optimizing the acquisition function substantially reduces computational time.
    \item An upper bound on the cumulative regret in terms of length scales that has a sublinear growth rate. To our knowledge, we are the first to characterize a sublinear regret bound for high dimensional BO using variable selection approach.
    \item Extensive experiments comparing our algorithm with various high-dimensional BO methods.
\end{itemize}

\section{Bayesian optimization with GP}
We address the following global optimization problem:
$$\mathbf{x}^{\star} \in \underset{\mathbf{x} \in [0,1]^D}{\arg \max } f(\mathbf{x}).$$
where $f: [0,1]^D \rightarrow \mathbb{R}$ is a costly black-box function that neither has a known closed-form expression, nor accessible derivatives and $D$ large. We assume that all variables in $\mathbf{x}$ can be divided into \textit{important variables}, which are variables that have a significant impact on $f$, and \textit{unimportant variables} that have little impact.

Bayesian Optimization (BO) provides a principal framework for solving the global optimization problem. The standard BO routine consists of two key steps: 1) Estimating the black-box function based on a handful number of observations.
2) Maximizing an acquisition function to select the next evaluation point by balancing exploration and exploitation. For the first step, Gaussian processes (GP) are a popular choice due to their tractability for posterior and predictive distributions. The GP model is defined as \( f(\mathbf{x}) = \mathcal{GP}(m(\mathbf{x}), k(\mathbf{x}, \mathbf{x}')) \), where \( m(\mathbf{x}) \) and \( k(\mathbf{x}, \mathbf{x}') \) represent the mean and covariance (or kernel) functions, respectively. Commonly used covariance functions include squared exponential (SE) kernel, and Mate\'rn kernel defined as
\begin{align}
k_{SE}^\psi(\mathbf{x}, \mathbf{x}')=\sigma_{k}^2 \exp \left\{-\frac{1}{2} \sum \rho_i\left(x_i-x'_i\right)^2\right\}
\label{kernel1}
\end{align}
\begin{align}
    k_{Matern}^\psi(\mathbf{x}, \mathbf{x}') &= \sigma^2_{k} \frac{2^{1-\nu}}{\Gamma(\nu)}\left(\sqrt{\sum \rho_i\left(x_i-x'_i\right)^2} \sqrt{2 \nu}\right)^\nu \nonumber \\
    &B_\nu\left( \sqrt{\sum \rho_i\left(x_i-x'_i\right)^2} \sqrt{2 \nu}\right)
\label{kernel2}
\end{align}
where $\rho_i$, $\forall i=1, \ldots, D$ are the inverse squared lengthscales (IDL). We use $\psi$ to collectively denote all the hyperparameters, i.e., $\psi=\left\{\rho_{1: D}, \sigma_k^2\right\}$. For a regression problem $f: [0,1]^D \rightarrow \mathbb{R}$, the joint density of a GP takes the form. 
Given a set of observations \( \mathbb{D}_t = \{(\mathbf{x}^i, y^i)\}_{i=1}^t \), the predictive distributions follow Gaussian distribution as:
\[ P(f_{t+1} \mid \mathbb{D}_t, \mathbf{x}) = \mathcal{N}(\mu_{t+1}(\mathbf{x}), \sigma_{t+1}^2(\mathbf{x})) \]
where
\[ \mu_{t+1}(\mathbf{x}) = \mathbf{k}^T [\mathbf{K} + \sigma^2 \mathbf{I}]^{-1} \mathbf{y} + m(\mathbf{x}) \]
\[ \sigma_{t+1}^2(\mathbf{x}) = k(\mathbf{x}, \mathbf{x}) - \mathbf{k}^T [\mathbf{K} + \sigma^2 \mathbf{I}]^{-1} \mathbf{k} \]
where $\mathbf{k}=\left[k\left(\mathbf{x}, \mathbf{x}_1\right), \ldots, k\left(\mathbf{x}, \mathbf{x}_t\right)\right]^T$, $\mathbf{K}=\left[k\left(\mathbf{x}_i, \mathbf{x}_j\right)\right]_{1 \leq i, j \leq t}$, $\mathbf{y}=\left[y_1, \ldots, y_t\right]^T$. 

For the next step, the acquisition functions are designed to trade-off between exploration of the search space and exploitation of the current promising region. Popular acquisition functions include Expected Improvement \citep{movckus1975bayesian,tran-the2}, \citep{tran-the1} and GP-UCB \citep{Srinivas2009InformationTheoreticRB}. A GP-UCB is defined  at iteration $t+1$  as
$$
a_{t+1}(\mathbf{x})=\mu_t(\mathbf{x})+\sqrt{\beta_{t+1}} \sigma_t(\mathbf{x})
$$
where $\beta_{t+1}$ is a parameter to balance exploration and exploitation. 

\section{The Proposed Algorithm LassoBO}
Due to space limitations, a review of related work is provided in Section 7 of the appendix. In this section, we introduce a novel variable selection algorithm for high-dimensional BO. Our main idea is to iteratively categorize all variables into important and unimportant groups and then construct an effective search space where the optimization at the acquisition step is performed only on important variables. Compared to existing variable selection algorithms, our novelty lies in (1) selecting important variables based on the inverse length scale and in (2) constructing an effective search region using multiple subspaces. We theoretically show that using these strategies will achieve an improved regret bound. We summarize in Algorithm \ref{alg_LassoBO} and present all the details below.
\begin{algorithm}[tb]
\caption{LassoBO} \label{alg_LassoBO}
\begin{algorithmic}[1]
\STATE \textbf{Input}: An initial dataset $\mathbb{D}_0$; the number of subspaces in iteration $t$: $M_t$
\STATE The best sampled point $\mathbf{x}^{0,+} = \arg \max_{\mathbf{x} \in \mathbb{D}_0} f(\mathbf{x})$
\FOR{$t = 1, 2 \dots T$}
    \STATE Fit a GP using $\mathbb{D}_{t-1}$ on all $D$ dimensions.
    \STATE Optimizing Eq. (\ref{eq:optimizing_Ut}) to get important variables  $\textcolor{blue}{I_t} \subset [D]$ and unimportant variables $\textcolor{orange}{[D]\setminus I_t}$
    \STATE Sample uniformly at random $M_t$ unimportant components $\mathcal{Z}_t = \{\mathbf{z}^i_{\textcolor{orange}{[D]\setminus I_t}} \in [0,1]^{D-|I_t|} \}_{i \in [M_t]}$ 
    \STATE Construct a search space $\mathcal{X}_t = [ \textcolor{blue}{\mathcal{X}_{I_t}}, \textcolor{orange}{\mathcal{X}_{[D]\setminus I_t}}]$ where $\textcolor{blue}{\mathcal{X}_{I_t}}=[0,1]^{|I_t|},\quad \textcolor{orange}{\mathcal{X}_{[D]\setminus I_t}} = \mathcal{Z}_t \cup \{\mathbf{x}^{t-1,+}_{[D]\setminus I_t}\}$
    \STATE Select  $\mathbf{x}^t = \text{argmax}_{ \mathbf{x} \in  \mathcal{X}_t } a_t(\mathbf{x}) $
    \STATE Evaluate $y^t = f(\mathbf{x}^t)$ 
    \STATE Update dataset $\mathbb{D}_t = \mathbb{D}_{t-1} \cup \{\mathbf{x}^t,y^t\}$
    \STATE Update the best point $\mathbf{x}^{t,+} = \arg \max_{\mathbf{x} \in \mathbb{D}_t} f(\mathbf{x}) $
    \ENDFOR
\end{algorithmic}
\end{algorithm}

\begin{figure*}
\vspace{-1pt}
    \centering
    \includegraphics[width = 1\textwidth]{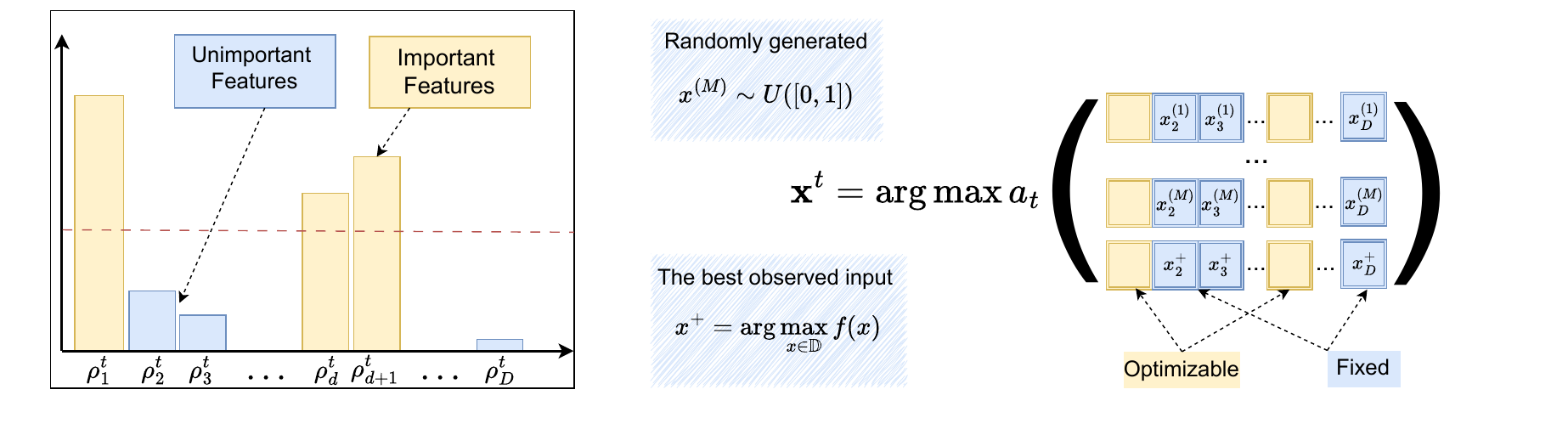}
    \vspace{-15pt}
    \caption{A overview of the LassoBO. \textbf{1)} Estimating the importance of dimensions by finding the sparse estimate of \(\rho\) to classify the dimensions into two categories: ``important" and ``unimportant" dimensions. \textbf{2)} The acquisition optimization will be performed on the important spaces (yellow) while the random (explore) and imputation from best found input (exploit) will be used to generate candidates for the unimportant space (blue).}
    \label{fig:overview}
    \vspace{-8pt}
\end{figure*}


\subsection{Variable Importance}
Our method is based on the natural assumption that some variables are more ``important" than others. A variable is important if it has a great impact on the output value of $f$ and vice versa. Our following theorem establishes  a strong link between the impact of each dimension $i$ toward $f$ via the length scale variable $\rho_i$.
\begin{theorem}
\label{var_imp}
Let $f$ be a sample from GP with SE kernel or Mate\'rn kernel. Given $L > 0$, for any $\mathbf{x}\in \left[0,1\right]^D$, we have that
\begin{itemize}
    \item for SE kernel, $
P\left(\left|\frac{\partial f}{\partial x_i}\right| \leq L \sqrt{\rho_i}\right) \geq  1 - e^{-\frac{L^2}{2 \sigma_k^2}}
$
\item for Mate\'rn kernel ($\nu = p + \frac{1}{2} > 1$), $
P\left(\left|\frac{\partial f}{\partial x_i}\right| \leq L \sqrt{\rho_i}\right) \geq  1 - e^{-\frac{L^2}{2 \sigma_k^2} \frac{2p-1}{2p+1}}
$
\end{itemize} 
\end{theorem}
A detailed proof of Theorem \ref{var_imp} is provided in the Appendix \ref{proofB}. In our Theorem, a larger $\rho_i$ indicates that the input dimension $i$ will influence the function $f$ more, i.e., $\left|\frac{\partial f}{\partial x_i}\right|$ fluctuates more. In contrast, the smaller $\rho_i$, the less the impact of $x_i$ on $f$. This suggests that we can measure the variable importance on the basis of the inverse length scale $\rho$. In practice, the true inverse length scales are unknown. However, we can estimate the true inverse length scales by maximizing the logarithmic marginal likelihood given the data \citep{williams2006gaussian}.
 
\paragraph{Estimating Inverse Length Scale via Lasso Maximum Likelihood.}
The negative GP log marginal likelihood w.r.t. the hyperparameter $\psi$ is written as:
\begin{align*}
&-\log p_\psi( \mathbf{y} \mid \mathbf{X})  = -\log \mathcal{N}\left(\mathbf{y} \mid \mathbf{0}, \mathbf{K} + \sigma^2 \mathbf{I}_N \right)  \nonumber \\
     &= -\frac{1}{2} \mathbf{y}^T\left(\mathbf{K}+\sigma^2 \mathbf{I}_N\right)^{-1} \mathbf{y}-\frac{1}{2} \log \left|\mathbf{K}+\sigma^2 \mathbf{I}_N\right| + \text{const}.
\end{align*}
However, obtaining an accurate estimate requires a large number of observations relative to the function's dimensionality \citep{ha2023provably}, making it challenging for high-dimensional Bayesian optimization problems. We utilize a Lasso method to estimate $\{\rho_i\}_{i \in [D]}$ when dealing with intrinsically low-effective dimensional functions in a high-dimensional space.
The general idea in Lasso \citep{tibshirani1996regression} is to utilize the $\text{L}_1$ regularization as a penalty term that encourages the number of active dimensions to be small. In particular, regularization of $\text{L}_1$ will force the less important parameters to be zero so that the parameter vector is sparse \citep{zheng2003statistical,tibshirani1996regression}. Despite being simple, Lasso has been successful in various feature selection tasks \citep{massias2018celer, coad2020catching}.  In our proposed method, we can derive the lasso estimate as the Bayes posterior mode by introducing the function prior in detail as follows:
\begin{minipage}{0.49\textwidth}
\begin{align}
& \sigma_{k}^2 \sim \mathcal{U} \left(0,10^2\right) \\
& \mathbf{f} \sim \mathcal{N}\left(\mathbf{0}, \mathbf{K}\right) \text { with } \psi=\left\{\rho_{1: d}, \sigma_k^2\right\} \\
& \rho_i \sim \mathcal{EXP}(\lambda) \quad \text { for } i \in [D] \label{eq_EXP_dist}\\
& \mathbf{y} \sim \mathcal{N}\left(\mathbf{f}, \sigma^2 \right)
\end{align}
\end{minipage}
where $\mathcal{U}(0, 10^2)$ denotes the continuous uniform distribution, $\mathcal{EXP}(\lambda)$ is the exponential distribution, i.e., $p(\rho_i \mid \lambda) = \lambda \exp{\{-\lambda \rho_i\}}, \quad \forall \rho_i >0$.

Given the connection from the variable important to $\rho$ in Theorem \ref{var_imp}, we learn the maximum posterior estimate of $\rho$. This is achieved by adding the negative log marginal likelihood to the sum of the absolute values of $\rho$ involving $\text{L}_1$ regularization, as shown below
\begin{align}
U_t(\psi)  = & -\log \mathcal{N}\left(\mathbf{y} \mid \mathbf{0}, \mathbf{K}+\sigma^2 \mathbf{I}_N\right)+\lambda \sum_{i \in [D]} \left|\rho_i\right| \nonumber \\
 = & -\frac{1}{2} \mathbf{y}^T\left(\mathbf{K}+\sigma^2 \mathbf{I}_N\right)^{-1} \mathbf{y}-\frac{1}{2} \log \left|\mathbf{K}+\sigma^2 \mathbf{I}_N\right| \nonumber \\
&+\lambda \sum_{i \in [D]} \left|\rho_i\right|+\text { const } \label{eq:optimizing_Ut}
\end{align}
where the regularization $\lambda$ is the rate parameter  $\mathcal{EXP}(\lambda)$ in Eq. (\ref{eq_EXP_dist}) and $\left|\mathbf{K}+\sigma^2 \mathbf{I}_N\right|$ is the determinant of the matrix. $L_1$ regularization is effective in cases where the function $f$ has few important variables because it promotes the sparsity of $\rho$, reflecting the impact of variables on the function.

We optimize Eq. (\ref{eq:optimizing_Ut}) using gradient-based optimization in which the derivatives are given below:
$$
\begin{aligned}
\frac{\partial U_t}{\partial \rho_j}  & =\frac{1}{2} \mathbf{y}^{\top} \mathbf{K}^{-1} \frac{\partial \mathbf{K}}{\partial \rho_j} \mathbf{K}^{-1} \mathbf{y}-\frac{1}{2} \operatorname{tr}\left(\mathbf{K}^{-1} \frac{\partial \mathbf{K}}{\partial \rho_j}\right) + \lambda \\
& =\frac{1}{2} \operatorname{tr}\left(\left(\boldsymbol{\alpha} \boldsymbol{\alpha}^{\top}-\mathbf{K}^{-1}\right) \frac{\partial \mathbf{K}}{\partial \rho_j}\right) + \lambda
\end{aligned}
$$
$$
\begin{aligned}
&\frac{\partial  U_t }{\partial \sigma_{k}}=\frac{1}{2} \mathbf{y}^{\top} \mathbf{K}^{-1} \frac{\partial \mathbf{K}}{\partial \sigma_{k}} \mathbf{K}^{-1} \mathbf{y}-\frac{1}{2} \operatorname{tr}\left(\mathbf{K}^{-1} \frac{\partial \mathbf{K}}{\partial \sigma_{k}}\right) \\
& =\frac{1}{2} \operatorname{tr}\left(\left(\boldsymbol{\alpha} \boldsymbol{\alpha}^{\top}-\mathbf{K}^{-1}\right) \frac{\partial \mathbf{K}}{\partial \sigma_{k}}\right) \text { where } \boldsymbol{\alpha}=\mathbf{K}^{-1} \mathbf{y} .
\end{aligned}
$$
\paragraph{Determining Important Variables.}
Let $\{\rho^{t}_i \}_{i\in [D]}$ be the solutions of the optimization problem in Eq. (\ref{eq:optimizing_Ut}) representing the importance of the variables $x_i$ at each iteration $t$.
 Based on $\{\rho^{t}_i \}_{i\in [D]}$, we can 
categorize variables into important and unimportant sets. There are various classification methods. A simple yet efficient approach is to consider a variable \(x_i\) as important if \(\rho_i^{t}\) is greater than the average value of $\{ \rho_i^{t}\}_{i \in [D]}$. The remaining variables are then classified as ``unimportant". This process produces the set of indices \(I_t\) representing the important variables and the number of the important variables $d_t = |I_t|$.


\subsection{Variable-Important Search Space}
We construct an effective search region in which the optimization of the acquisition function is performed. The main idea is to optimize the acquisition function on only the important variables while keeping the unimportant variables fixed. By doing this, we reduce the optimization to a subspace rather than the entire high-dimensional space, which is computationally expensive. 

Let denote $\mathbf{x} = \left[\mathbf{x}_{I_t}, \mathbf{x}_{[D] \setminus I_t} \right] \in [0,1]^D $, where $\mathbf{x}_{I_t}$ are important and $\mathbf{x}_{[D] \setminus I_t}$ are unimportant variables.

\paragraph{Existing methods for imputing unimportant variables.}
Existing methods \citep{li2018high, song2022monte} have considered two efficient solutions to select $\mathbf{x}_{[D] \setminus I_t}$: (1) copy the value of $\mathbf{x}_{[D] \setminus I_t}$ from the best input observed so far:
$\mathbf{x}^{t,+} = \text{argmax}_{\mathbf{x} \in \mathbb{D}_t} f(\mathbf{x}^i)$. Although the original vector of $\mathbf{x}^{t,+}$ includes both important and unimportant components, we can extract the corresponding unimportant variables denoted as $\mathbf{x}^{t,+}_{[D]\setminus I_t}$. However, this pure exploitative can get stuck in a local optimum; and (2) use a random value in the search space $\mathbf{x}^t_{[D] \setminus I_t} \sim \mathcal{U}(\left[0, 1\right] ^{D-|I_t|})$ where $D-|I_t|$. Still, this solution is fully explorative and thus slow to converge. 

\paragraph{Explorative-exploitative imputation.}

To jointly analyze the merits of the aforementioned approaches, we make use of both to generate multiple different values of the unimportant component $\mathbf{x}_{[D] \setminus I_t}$ in explorative-exploitative ways.


Specifically, at iteration $t$, we generate $M_t +1$ different values of $\mathbf{x}_{[D] \setminus I_t}$ including (i) a vector from copying the best function value as above described, denoted by $\mathbf{x}^{t,+}_{[D]\setminus I_t}$, and (ii) $M_t$ different vectors by uniform sampling, denoted by
$\mathcal{Z} = \left\{\mathbf{z}^i_{[D] \setminus I_t} \sim \text{Uniform} [0,1]^{D-d_t} \right\}_{i \in [M_t]}$. 

Using $M_t +1$ different values of $\mathbf{x}_{[D] \setminus I_t}$, we construct a set of subspaces $\mathcal{X}_t = [\textcolor{blue}{\mathcal{X}_{I_t}}, \textcolor{orange}{\mathcal{X}_{[D]\setminus I_t}}]$ where $\textcolor{blue}{\mathcal{X}_{I_t}}=[0,1]^{|I_t|}$ and $\textcolor{orange}{\mathcal{X}_{[D]\setminus I_t}} = \mathcal{Z}_t \cup \{\mathbf{x}^{t-1,+}_{[D]\setminus I_t}\}$.
    We use this set as a search space to maximize acquisition optimization. In particular, we optimize for the important space $\textcolor{blue}{\mathcal{X}_{I_t}}$ while keeping the unimportant space $\textcolor{orange}{\mathcal{X}_{[D]\setminus I_t}}$ fixed.   
    

Importantly, we show that maximizing the acquisition function using the defined search space can result in better regret by quantifying a valid value of $M_t$. 
\section{Theoretical Analysis}
We generalize Theorem \ref{var_imp} to the following assumption \ref{as:1} for ease of analysis, since the constants $a,b$ may differ for each kernel choice.
\begin{assumption}
    \label{as:1} The function $f$ is a GP sample path. For some $a, b>0$, given $L>0$, the partial derivatives of $f$ satisfy that $\forall i \in[D]$,
$$
P\left(\sup _{\boldsymbol{x} \in \mathcal{X}}\left|\partial f / \partial x_i\right|<\sqrt{\rho_i} L\right) \geq 1-a e^{-(L / b)^2} .
$$
\end{assumption}
This assumption holds at least for SE
kernel and Mate\'rn kernel. Let $\boldsymbol{x}^* = \arg \max f(\boldsymbol{x})$ be an optimal solution. We derive the upper bound on the cumulative regret of the proposed algorithm $R_T=$ $\sum_{t=1}^T\left(f\left(\boldsymbol{x}^*\right)-f\left(\boldsymbol{x}^t\right)\right)$. Based on the theoretical bound, we show how the algorithm should be designed for $R_T$ to achieve sublinearity asymptotically. We summarize the main theoretical result and refer to the Appendix \ref{maintheorem} for the proofs.
\begin{theorem}
Set $\rho^* = \max_{i\in [D]} \rho_i$, $\tilde{d}_t = \max_{1\leq k \leq t} d_t$ and $\beta_t = 2\log(\frac{\pi^2t^2}{\delta}) + 2\tilde{d}_t\log(4\sqrt{\rho^*}b\tilde{d}_t\sqrt{\log(\frac{12Da}{\delta})}t^2)$ and $C = \max_{t \in [T]} 2b\sqrt{\log(\frac{2Da}{\delta})}(\Gamma(D-d_t + 1))^{\frac{1}{D-d_t} }\log(\frac{6}{\delta})$, where $\Gamma(k) = k!$. Pick $\delta \in (0,1)$, the cumulative regret of the proposed algorithm is bounded by
\begin{align}
    R_T \le & \sqrt{\beta_TC_1T\gamma_T}  + \frac{\pi^2}{6} \nonumber \\ &  + C \sqrt{\left( \sum^T_{t=1} \sum_{i\in [D ] \setminus I_t} \rho_i  \right) \left(\sum^T_{t=1}(\frac{1}{M_t})^{\frac{2}{D-d_t}}\right)} 
\end{align}
with probability $\ge 1 -\delta$, where $C_1 = 8/\log(1 + \sigma^2)$, $\gamma_T$ is the
maximum information gain about the function $f$ from any set of observations of size $T$.
\label{theo:main}
\end{theorem}

\begin{corollary}
\label{main_result}
Under the same conditions as in Theorem \ref{theo:main} and set $M_t = \lceil \sqrt[n]{t} \rceil$ where $n \in \mathbb{N}, n > 1$, the cumulative regret of the proposed algorithm is bounded by
\begin{align}
R_T \le & \sqrt{\beta_TC_1T\gamma_T}  + \frac{\pi^2}{6} \nonumber \\ & + C \sqrt{c_0} \sqrt{\left( \sum^T_{t=1} \sum_{i\in [D ] \setminus I_t} \rho_i \right)} T^{\frac{1}{2} - \frac{1}{n(D-1)}},
\end{align}
with probability $\ge 1 -\delta$, where $C_1 = 8/\log(1 + \sigma^2)$, $\gamma_T$ is the
maximum information gain about function $f$ from any set of observations of size $T$, $c_0 > \frac{n(D-1)}{n(D-1)-2}$.
\end{corollary}


Corollary \ref{main_result} provides an upper bound on the cumulative regret $R_T$ for our proposed algorithm. However, in our proposed algorithm, we select only a subset of variables (i.e., $I_{d_t} << [D]$) to optimize. Therefore, $\beta_T$ reduces to $\mathcal O(\tilde{d})$, where $\tilde{d} = \text{max}_{1 \le t \le T} |I_{d_t}|$ and $\tilde{d} << D$ instead of $\mathcal O(D)$ as in standard regret bound \citep{Srinivas2009InformationTheoreticRB}. However, using variable selection will make the regret bound loose (worse) by the additional factor $ C \sqrt{c_0} \sqrt{\left( \sum^T_{t=1} \sum_{i\in [D ] \setminus I_t} \rho_i \right)} T^{\frac{1}{2}  - \frac{1}{n(D-1)}}$. The component $\sum^T_{t=1} \sum_{i \in [D]\setminus I_t} \rho_i$ relates the inverse length scales of variables in set $[D] \setminus I_t$ which are classed as unimportant variables as shown in Section 3.1. Finding a valid set $I_t$ can make this component significantly small. We provide an experimental estimate of this set as in Figure 5. In the worst case, this component still can be upper bounded by $T(\sum_{i \in [D]} \rho_i)$. 
Thus, the additional factor is still bounded by $\mathcal O(T^{1 - \frac{1}{n(D-1)}})$ which is sublinear in $T$ for every $n \ge 1$ and $D \ge 4$ thanks to the choice $M_t = \lceil \sqrt[n]{t} \rceil$ which guarantees the enough exploration. Thus, $\lim_{T \rightarrow \infty} \frac{R_T}{T} = 0$. Compared to recent high dimensional BO methods using variable selection (e.g., \citep{ li2018high,shen2021computationally,song2022monte}), we are \textit{the first} to characterize a sublinear regret bound. Moreover, in \cite{li2018high} and \cite{shen2021computationally}, a regret bound analysis is performed using $d$ fixed important variables. In contrast, our regret-bound analysis is performed with an adaptive set of $d_t$ important variables which are more flexible.

\begin{figure*}[h!]
    \centering
    \vspace{-3pt}
    {\includegraphics[width=0.48\textwidth]{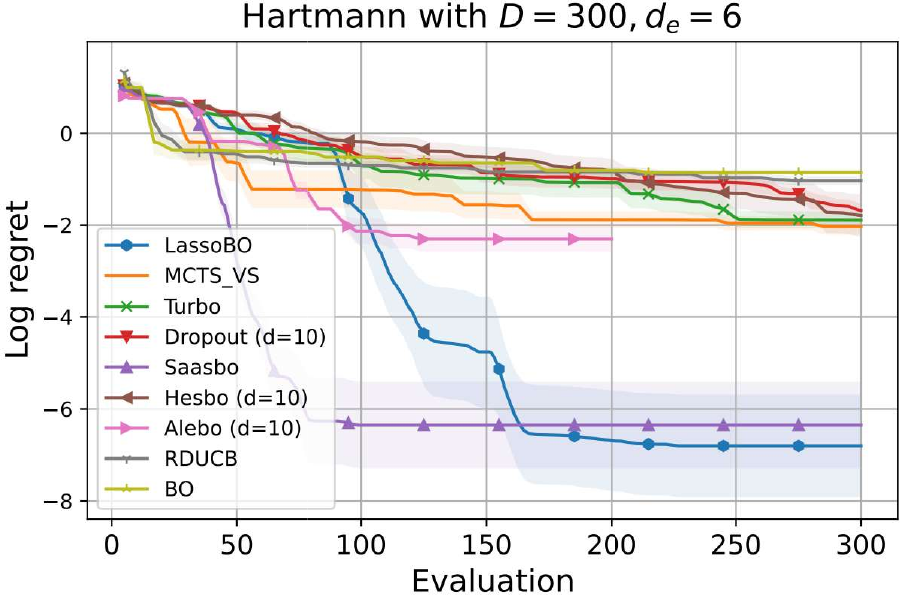}}
    {\includegraphics[width=0.48\textwidth]{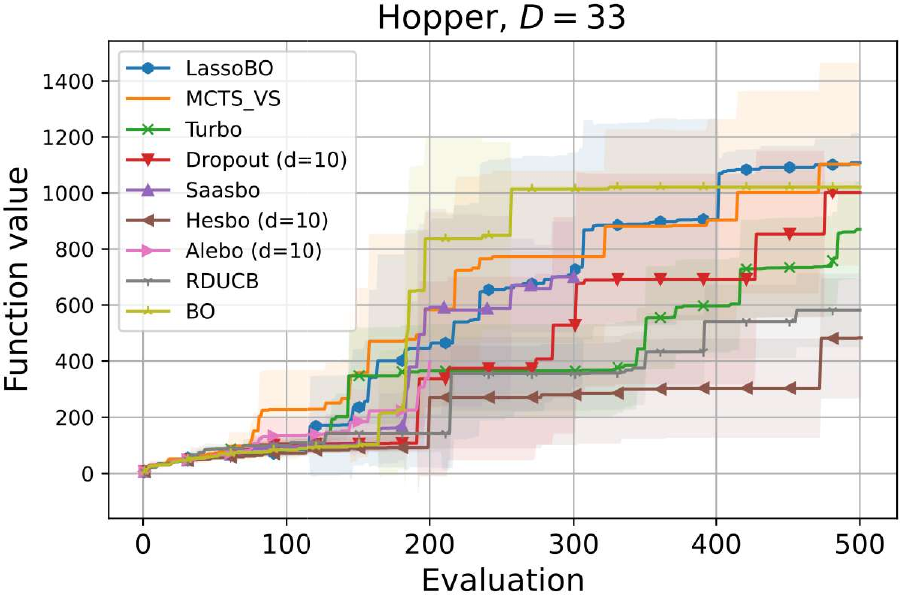}}

    
    {\includegraphics[width=0.48\textwidth]{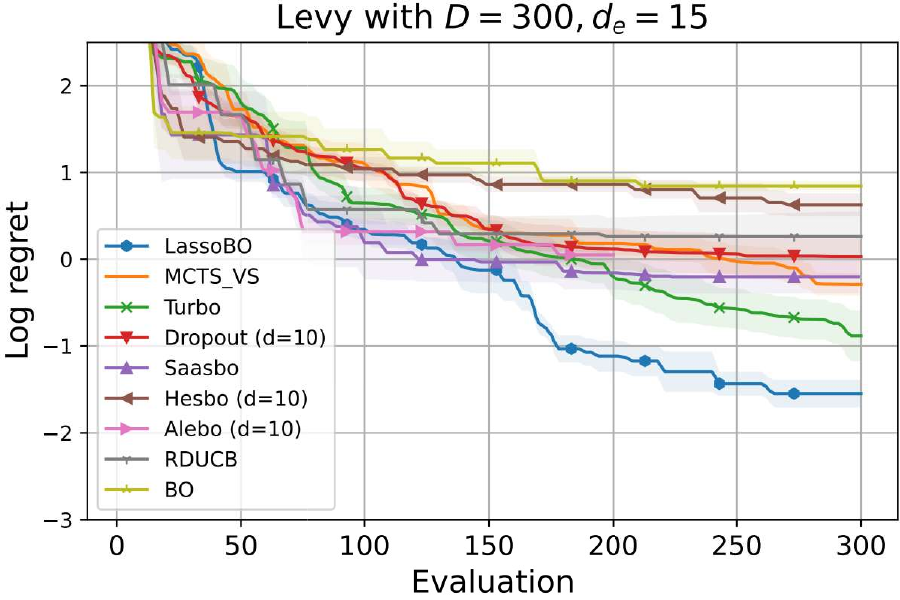}}
    {\includegraphics[width=0.48\textwidth]{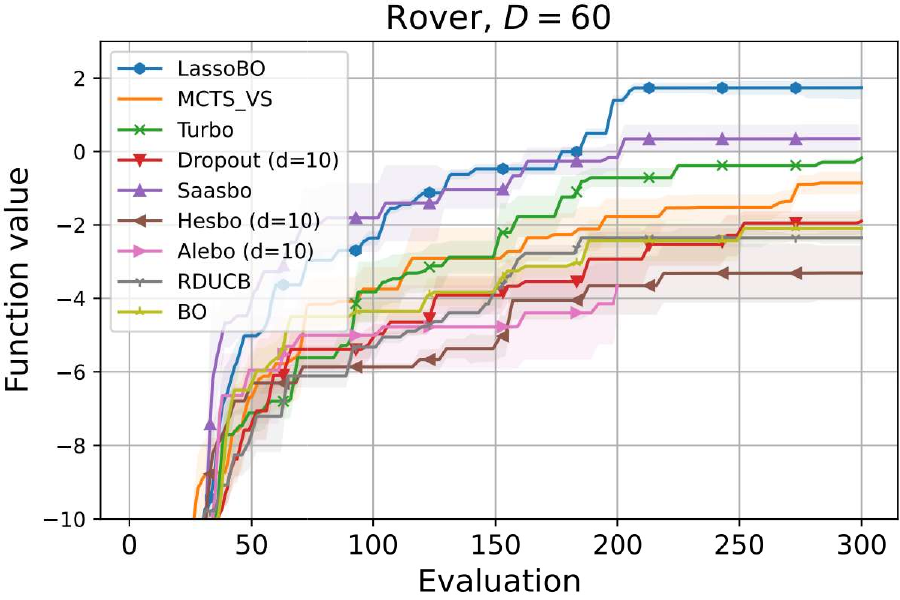}}
    
    { \includegraphics[width=0.48\textwidth]{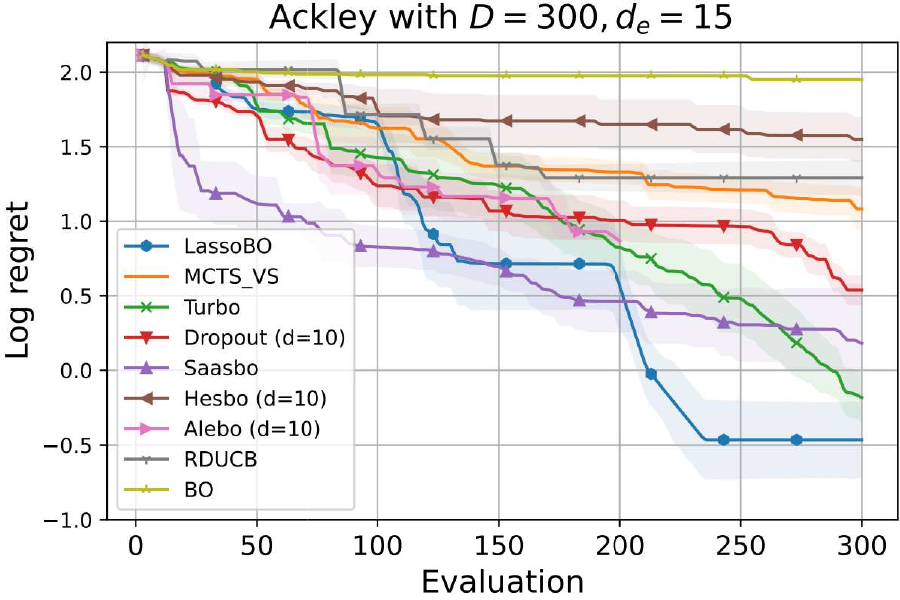}}
    {\includegraphics[width=0.48\textwidth]{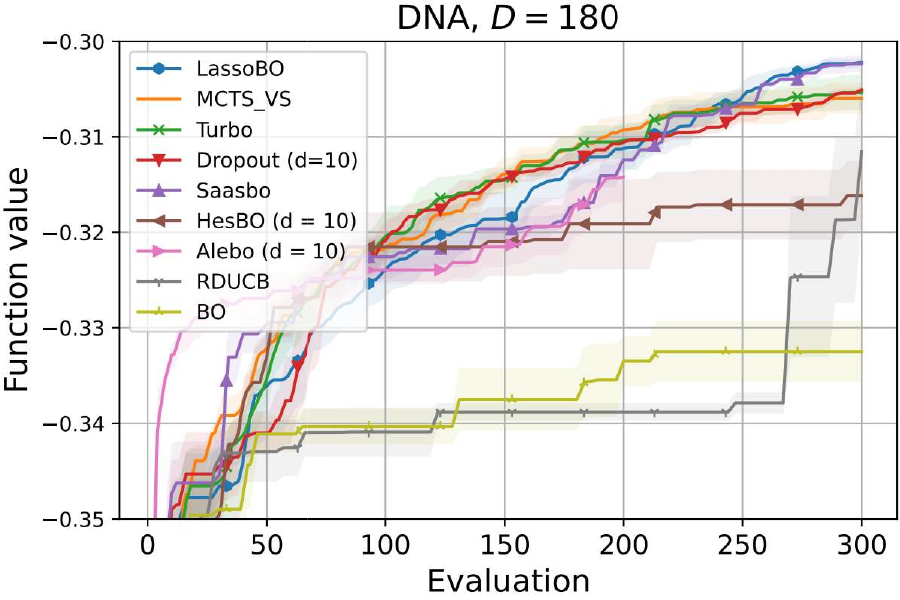}}
    \vspace{-8pt}
    \caption{Comparison with the BO baselines on the high dimensional optimization tasks including the benchmark functions (\textbf{Left}) and real-world applications (\textbf{Right}). Our proposed LassoBO outperforms the baselines by a wide margin. Here, \(d_e\) denotes the number of valid dimensions of the function. $d < d_e$ is the hyperparameter that determines the effective subspace dimension of the algorithm. In \textbf{Right}, we don't observe $d_e$ in advance.} 
    \label{fig:synthetic}
    \label{fig:realworld}
    \vspace{-5pt}
\end{figure*}

\paragraph{Computational Complexity.}
 The computation of each iteration depends on four components: computing variable importance, fitting a GP surrogate model, maximizing an acquisition function, and evaluating a sampled point. The first step of estimating the importance of the variable in the optimization process $U_t(\psi)$ requires $\mathcal{O} (t^3 + Dt^2)$. The computational complexity of fitting a GP model at iteration $t$ is $\mathcal{O}\left(t^3+t^2 D\right)$. Maximizing an acquisition function is related to the optimization algorithm. If we use the Quasi-Newton method to optimize UCB, the computational complexity is $\mathcal{O}\left(mM_t\left(t^2+t d_t+d_t^2\right)\right)$ \citep{nocedal1999numerical}, where $m$ denotes the Quasi-Newton's running rounds. 
  In our experiments, we use a small value of $M_t$ by fixing $n =3$. We use the same computation budget for the optimization of the acquisition function. Therefore, compared to MCTS\_VS \citep{song2022monte}, the computation complexity is the same.

\begin{figure*}
{\includegraphics[width=0.325\textwidth]{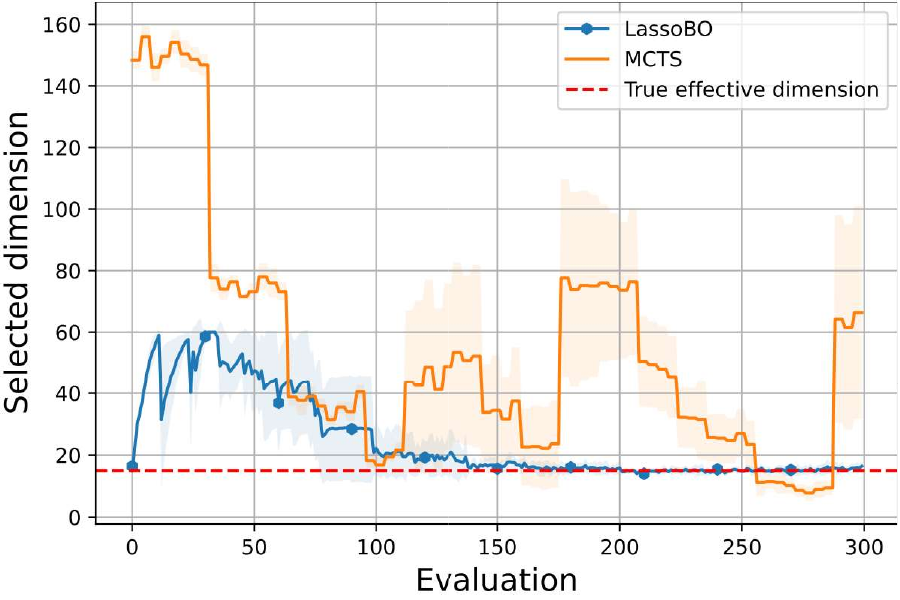}
    \label{fig:rb20300a}}
{\includegraphics[width=0.325\textwidth]{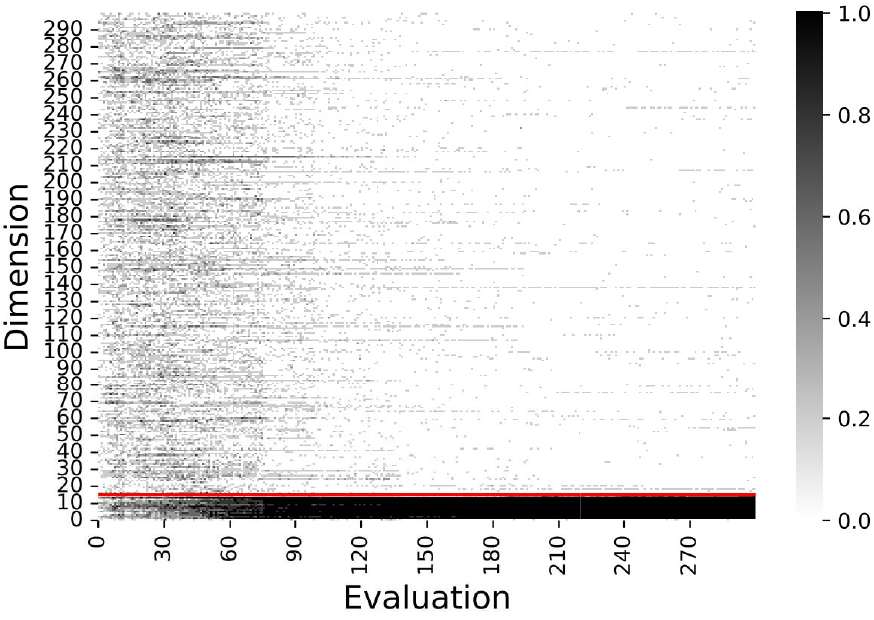}
    \label{fig:rb20300b}}
{\includegraphics[width=0.325\textwidth]{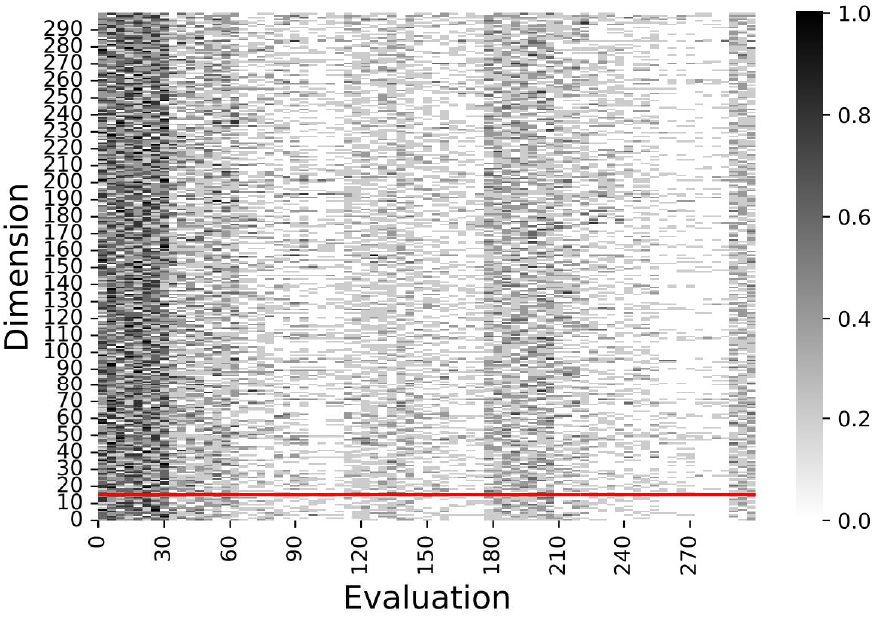}
    \label{fig:rb20300c}}
    \vspace{-15pt}
    \caption{We compare the variable selection between LassoBO and MCTS when performed on the Levy function($d_e=15, D= 300$). \textbf{Left:} We compare the number of selected variables through function evaluations in LassoBO and MCTS\_VS. \textbf{Middle:} We depict the selected dimensions for each evaluation in LassoBO. \textbf{Right: }We depict the selected dimensions for each evaluation in MCTS\_VS. }
    \vspace{-7pt}
    \label{fig:rb20300}
\end{figure*}

\begin{figure}
\centering
{\includegraphics[width=0.45\textwidth]{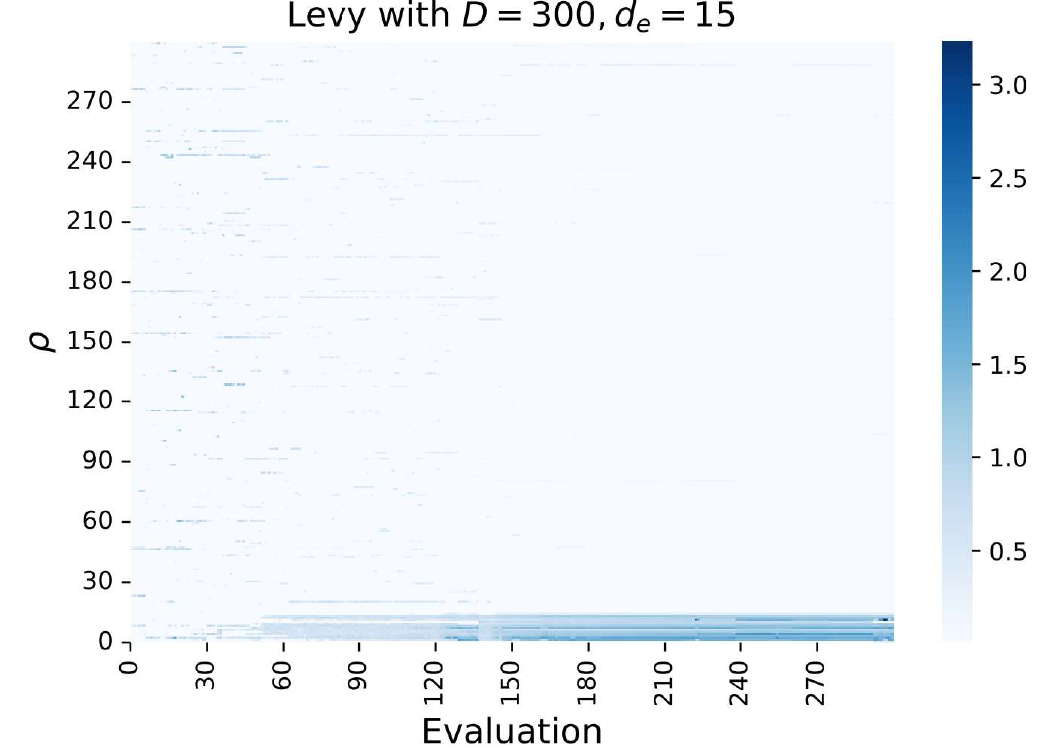}}        
    \vspace{-3pt}
    \caption{The estimated $\rho_i$ at different time step for Levy ($D = 300, d_e = 15$). We show that the algorithm can converge correctly to the true dimensions of $\rho_1,...\rho_{15}$, located at the bottom area of the plot.} 
    \label{fig:rho}
    \vspace{-5pt}
\end{figure}



\section{Experiment}
\subsection{The experimental setup}
\label{subsec:setup}
We evaluate the performance of LassoBO on different tasks, including synthetic functions, Rover problem, MuJoCo locomotion tasks and LASSOBENCH. We benchmark against MCTS\_VS \citep{song2022monte}, TurBO \citep{eriksson2019scalable} with one trust region, SAASBO \citep{eriksson2021high}, Dropout \citep{li2018high}, ALEBO \citep{letham2020re}, HeSBO \citep{nayebi2019framework}, RDUCB \citep{ziomek2023random} and BO \citep{Srinivas2009InformationTheoreticRB}. 
 In LassoBO, we default to choosing \(\lambda = 10^{-3}\) and \(M_t = \lceil \sqrt[3]{t} \rceil \) (i.e. $n=3$) for all experiments. For Dropout and embedding-based methods, we run Dropout, HesBO and ALEBO with $d = 10$. 
 The experiments are conducted on an AMD Ryzen 9 5900HX CPU @ 3.3GHz, and 16 GB of RAM. We observe that ALEBO is constrained by their memory consumption. The available hardware allows up to $250$ function evaluations for ALEBO for each run. Larger sampling budgets or higher target dimensions for ALEBO resulted in out-of-memory errors. Throughout all experiments, the number of initial observations is set at $30$ and each algorithm is repeated $10$ times with different initialization. Our code is available at \url{https://github.com/Fsoft-AIC/LassoBO}.
\subsection{Synthetic Functions}




We use Hartmann $(d_e = 6)$, Ackley $(d_e=15)$ and Levy $(d_e = 15)$ as the synthetic benchmark functions, and extend them to high dimensions by adding unrelated variables. For example, the original Hartmann has $(d_e = 6)$ dimensions, we append $294$ unimportant dimensions to create $D = 300$  dimension.
The variables affecting the value of $f$ are important variables. We evaluate the performance of each algorithm using a fixed budget of $T=300$ and measure the logarithmic distance to the true optimum, that is, $\ln {f(\mathbf{x}^*) - f(\mathbf{x}_t)}$.  As shown in Fig. \ref{fig:synthetic}, LassoBO performs the best in Levy followed by TurBO. LassoBO has moderate results in the initial iterations because only a small number of important variables are selected, but has the best results in later iterations when the important variables are correctly estimated. The reason for Turbo's inferior performance when compared to LassoBO is the greater difficulty in adjusting the confidence region as opposed to determining the variable significance in these settings. 
MCTS\_VS has moderate performance, indicating that the method needs more samples to be able to detect effective dimensions and to focus on optimization on these dimensions. ALEBO and HESBO performed inconsistently, having a relatively good performance with Hartmann but a poor performance with Ackley and Levy because the low-dimensional space is smaller than the effective dimension ($d< d_e$). RDUCB has moderate performance and BO has poor performance in these settings. SAASBO demonstrates superior performance on the Hartmann function, with LassoBO closely following as the second-best performer. However, with the increase in the effective dimensional ratio in Ackley and Levy ($D=300,d_e=15$), SAASBO is unable to benefit from the present active subspace. This is because the number of active dimensions in SAASBO is smaller than the valid dimensions, and the inactive dimensions, which are not controlled, can negatively impact performance.

\subsection{Real-World Problems}
 We further compare LassoBO with the baselines on real-world problems, including Rover \citep{wang2018batched}, Hopper and DNA \citep{nardi2022lassobench}. Rover problems is a popular benchmark in high-dimensional BO. Hopper is a robot locomotion task in MuJoCo \citep{todorov2012mujoco}, which is a popular black-box optimization benchmark. DNA is a benchmark from LASSOBENCH \citep{nardi2022lassobench} and has a low effective dimensionality. See details in Appendix \ref{sec:inforimplement}.
 
In Fig. \ref{fig:realworld}, LassoBO and SAASBO are the best performers on Rover. ALEBO and HESBO have poor performance with low dimensionality $d=10$, thus requiring a reasonable choice of $d$. In the MuJoCo tasks, most methods have large variance due to the aleatoric noise of $f$. Both LassoBO and MCTS\_VS obtain good performances due to focusing on identifying and optimizing only a subset of important variables. In DNA setting, LassoBO achieves a better solution at later evaluation than any other method in DNA setting. This is because  LassoBO can gradually identify correctly the important variables with sufficient observation that leads to better performance.
\subsection{Ablation Study}
\paragraph{Compare with MCTS\_VS.}
 We empirically compare their ability to select variables in the Levy function ($D=300, d_e = 15$) as shown in Fig. \ref{fig:rb20300}. The first subfigure indicates that the number of selected variables in LassoBO begins at a minimal level and progresses towards the number of effective dimensions. In contrast, MCTS\_VS starts with a relatively high number of selected variables and progressively decreases to a level smaller than the effective number of dimensions. In Fig. \ref{fig:rb20300}, the Middle presents the variables selected by LassBO and the Right by MCTS\_VS over iterations. The results show that LassoBO can correctly identify effective dimensions, while MCTS\_VS have selected many ineffective dimensions, implying the benefit of variable selection in LassoBO over MCTS\_VS.

\paragraph{Correlation between $\rho$ and variable importance}
We study the learning behavior of LassoBO in determining variable importance. We examine the evolution of $\rho$ following function evaluations. As illustrated in Fig \ref{fig:rho}, the value of $\rho_i$ exhibits a gradual increase after each function evaluation and eventually converges. This stabilization occurs after the 180th evaluation, with correctly identifying $10$ \textit{relevant} dimensions assigning substantial value compared to the remaining dimensions. 
Next, we consider the correlation between $\rho$ and variable importance when the number of function evaluations is large enough. In Fig. \ref{fig:rho-alpha}, we describe the value of $\rho$ and the coefficient $\alpha$ of the Sum Squares function $f(x) = \sum \alpha_i \mathbf{x}_i^2$, where $\alpha_i > 0, \forall i \in \{1 \dots D \}$. 
In Fig. \ref{fig:rho-alpha}, we clearly demonstrate a strong correlation between $\sqrt{\rho_i}$ and $\alpha_i$ across multiple dimensions.


 \begin{figure}
\centering
    \hspace{-2pt}
    {\includegraphics[width = 0.49 \textwidth]{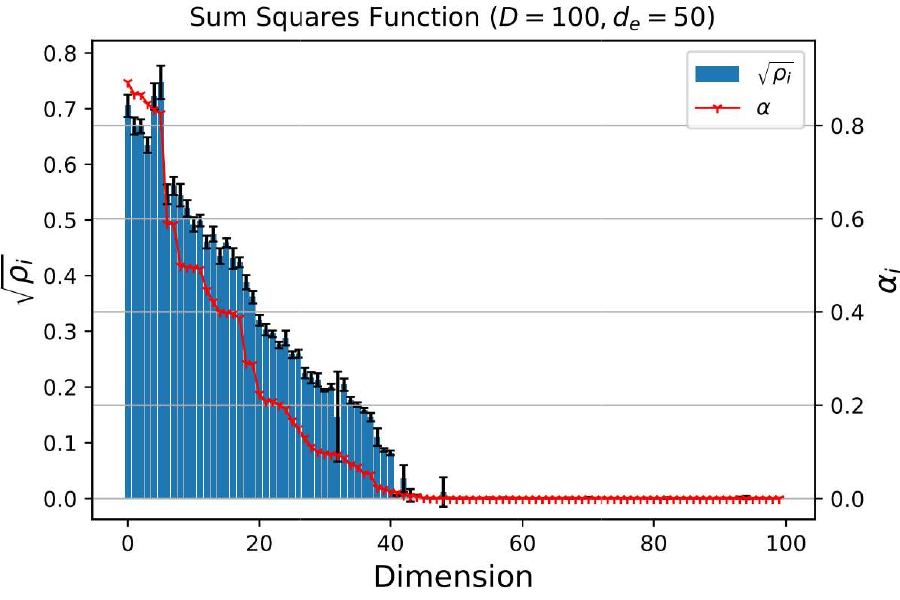}}
    \vspace{-7pt}
    \caption{Correlation between $\rho_i$. The results are averaged across $10$ iterations from $290$th to $300$th. The red line represents the dimension importance value, $\alpha$, the blue bars represent $\sqrt{\rho}$.}
    \label{fig:rho-alpha}
    \vspace{-6pt}
\end{figure}

\paragraph{Further Studies.}  We test our model performance by varying $D=100$, $D=300$ and $D=500$ while keeping $d_e = 15$ fixed in Appendix \ref{sec:addexperiment}. We also investigate employing different window sizes to estimate the value of $\rho$ at each iteration.

\section{Conclusion}
\label{sec:conlusion}

In this paper, we have presented a new BO method for high-dimensional settings. The proposed LassoBO estimates the variable-important score to partition the dimensions into important and unimportant groups. We then design a Variable Importance-based Searching Region to optimize the acquisition function effectively. We analyze our algorithm theoretically and show that regardless of the variable classification method used, our algorithm can achieve a sublinear growth rate for cumulative regret. We perform experiments for many high-dimensional optimization problems and show that our algorithm performs better than the existing methods given the same evaluation budget.


\bibliography{ref}

\begin{thebibliography}{}

\bibitem[Ament et~al., 2023]{ament2023unexpected}
Ament, S., Daulton, S., Eriksson, D., Balandat, M., and Bakshy, E. (2023).
\newblock Unexpected improvements to expected improvement for bayesian optimization.
\newblock {\em Advances in Neural Information Processing Systems}, 36:20577--20612.

\bibitem[Bergstra et~al., 2011]{bergstra2011algorithms}
Bergstra, J., Bardenet, R., Bengio, Y., and K{\'e}gl, B. (2011).
\newblock Algorithms for hyper-parameter optimization.
\newblock {\em Advances in neural information processing systems}, 24.

\bibitem[Binois and Wycoff, 2022]{binois2022survey}
Binois, M. and Wycoff, N. (2022).
\newblock A survey on high-dimensional gaussian process modeling with application to bayesian optimization.
\newblock {\em ACM Transactions on Evolutionary Learning and Optimization}, 2(2):1--26.

\bibitem[Chen et~al., 2012]{Chen2012JointOA}
Chen, B., Castro, R.~M., and Krause, A. (2012).
\newblock Joint optimization and variable selection of high-dimensional gaussian processes.
\newblock In {\em International Conference on Machine Learning}.

\bibitem[Chlebus, 2009]{chlebus2009approximate}
Chlebus, E. (2009).
\newblock An approximate formula for a partial sum of the divergent p-series.
\newblock {\em Applied Mathematics Letters}, 22(5):732--737.

\bibitem[Coad and Srhoj, 2020]{coad2020catching}
Coad, A. and Srhoj, S. (2020).
\newblock Catching gazelles with a lasso: Big data techniques for the prediction of high-growth firms.
\newblock {\em Small Business Economics}, 55(3):541--565.

\bibitem[Eriksson and Jankowiak, 2021]{eriksson2021high}
Eriksson, D. and Jankowiak, M. (2021).
\newblock High-dimensional {B}ayesian optimization with sparse axis-aligned subspaces.
\newblock In {\em Uncertainty in Artificial Intelligence}, pages 493--503. PMLR.

\bibitem[Eriksson et~al., 2019]{eriksson2019scalable}
Eriksson, D., Pearce, M., Gardner, J., Turner, R.~D., and Poloczek, M. (2019).
\newblock Scalable global optimization via local {B}ayesian optimization.
\newblock {\em Advances in neural information processing systems}, 32.

\bibitem[Frazier, 2018]{frazier2018tutorial}
Frazier, P.~I. (2018).
\newblock A tutorial on {B}ayesian optimization.
\newblock {\em arXiv preprint arXiv:1807.02811}.

\bibitem[Garnett, 2023]{garnett2023bayesian}
Garnett, R. (2023).
\newblock {\em Bayesian optimization}.
\newblock Cambridge University Press.

\bibitem[Ha et~al., 2023]{ha2023provably}
Ha, H., Nguyen, V., Tran-The, H., Zhang, H., Zhang, X., and Hengel, A. v.~d. (2023).
\newblock Provably efficient bayesian optimization with unknown gaussian process hyperparameter estimation.
\newblock {\em arXiv preprint arXiv:2306.06844}.

\bibitem[Han et~al., 2021]{han2021high}
Han, E., Arora, I., and Scarlett, J. (2021).
\newblock High-dimensional {B}ayesian optimization via tree-structured additive models.
\newblock In {\em Proceedings of the AAAI Conference on Artificial Intelligence}, volume~35, pages 7630--7638.

\bibitem[Hansen, 2016]{hansen2016cma}
Hansen, N. (2016).
\newblock The cma evolution strategy: A tutorial.
\newblock {\em arXiv preprint arXiv:1604.00772}.

\bibitem[Hoang et~al., 2018]{hoang2018decentralized}
Hoang, T.~N., Hoang, Q.~M., Ouyang, R., and Low, K.~H. (2018).
\newblock Decentralized high-dimensional {B}ayesian optimization with factor graphs.
\newblock In {\em Proceedings of the AAAI Conference on Artificial Intelligence}, volume~32.

\bibitem[Husain et~al., 2023]{husain2023distributionally}
Husain, H., Nguyen, V., and van~den Hengel, A. (2023).
\newblock Distributionally robust bayesian optimization with $\phi$-divergences.

\bibitem[Hvarfner et~al., 2024]{hvarfner2024vanilla}
Hvarfner, C., Hellsten, E.~O., and Nardi, L. (2024).
\newblock Vanilla bayesian optimization performs great in high dimension.
\newblock {\em arXiv preprint arXiv:2402.02229}.

\bibitem[Hvarfner et~al., 2022]{hvarfner2022pi}
Hvarfner, C., Stoll, D., Souza, A., Lindauer, M., Hutter, F., and Nardi, L. (2022).
\newblock $\pi$bo: Augmenting acquisition functions with user beliefs for bayesian optimization.
\newblock In {\em Tenth International Conference of Learning Representations, ICLR 2022}.

\bibitem[Kandasamy et~al., 2015]{kandasamy2015high}
Kandasamy, K., Schneider, J., and P{\'o}czos, B. (2015).
\newblock High dimensional {B}ayesian optimisation and bandits via additive models.
\newblock In {\em International conference on machine learning}, pages 295--304. PMLR.

\bibitem[Klein et~al., 2017]{klein2017fast}
Klein, A., Falkner, S., Bartels, S., Hennig, P., and Hutter, F. (2017).
\newblock Fast {B}ayesian optimization of machine learning hyperparameters on large datasets.
\newblock In {\em Artificial intelligence and statistics}, pages 528--536. PMLR.

\bibitem[Letham et~al., 2020]{letham2020re}
Letham, B., Calandra, R., Rai, A., and Bakshy, E. (2020).
\newblock Re-examining linear embeddings for high-dimensional {B}ayesian optimization.
\newblock {\em Advances in neural information processing systems}, 33:1546--1558.

\bibitem[Li et~al., 2017]{li2018high}
Li, C., Gupta, S., Rana, S., Nguyen, V., Venkatesh, S., and Shilton, A. (2017).
\newblock High dimensional bayesian optimization using dropout.
\newblock In {\em Proceedings of the 26th International Joint Conference on Artificial Intelligence}, pages 2096--2102.

\bibitem[Lizotte et~al., 2007]{lizotte2007automatic}
Lizotte, D.~J., Wang, T., Bowling, M.~H., Schuurmans, D., et~al. (2007).
\newblock Automatic gait optimization with {G}aussian process regression.
\newblock In {\em IJCAI}, volume~7, pages 944--949.

\bibitem[Lu et~al., 2022]{lu2022additive}
Lu, X., Boukouvalas, A., and Hensman, J. (2022).
\newblock Additive gaussian processes revisited.
\newblock In {\em International Conference on Machine Learning}, pages 14358--14383. PMLR.

\bibitem[Marco et~al., 2017]{marco2017virtual}
Marco, A., Berkenkamp, F., Hennig, P., Schoellig, A.~P., Krause, A., Schaal, S., and Trimpe, S. (2017).
\newblock Virtual vs. real: Trading off simulations and physical experiments in reinforcement learning with {B}ayesian optimization.
\newblock In {\em 2017 IEEE International Conference on Robotics and Automation (ICRA)}, pages 1557--1563. IEEE.

\bibitem[Massias et~al., 2018]{massias2018celer}
Massias, M., Gramfort, A., and Salmon, J. (2018).
\newblock Celer: a fast solver for the lasso with dual extrapolation.
\newblock In {\em International Conference on Machine Learning}, pages 3315--3324. PMLR.

\bibitem[Mo{\v{c}}kus, 1975]{movckus1975bayesian}
Mo{\v{c}}kus, J. (1975).
\newblock On {B}ayesian methods for seeking the extremum.
\newblock In {\em Optimization techniques IFIP technical conference: Novosibirsk, July 1--7, 1974}, pages 400--404. Springer.

\bibitem[Mutny and Krause, 2018]{mutny2018efficient}
Mutny, M. and Krause, A. (2018).
\newblock Efficient high dimensional {B}ayesian optimization with additivity and quadrature fourier features.
\newblock {\em Advances in Neural Information Processing Systems}, 31.

\bibitem[Nardi et~al., 2022]{nardi2022lassobench}
Nardi, L., Gramfort, A., Salmon, J., and Sehic, K. (2022).
\newblock Lassobench: A high-dimensional hyperparameter optimization benchmark suite for lasso.
\newblock In {\em 1st International Conference on Automated Machine Learning (AutoML)}.

\bibitem[Nayebi et~al., 2019]{nayebi2019framework}
Nayebi, A., Munteanu, A., and Poloczek, M. (2019).
\newblock A framework for {B}ayesian optimization in embedded subspaces.
\newblock In {\em International Conference on Machine Learning}, pages 4752--4761. PMLR.

\bibitem[Nguyen and Tran, 2024]{Nguyen_Tran}
Nguyen, Q.-A.~H. and Tran, T.~H. (2024).
\newblock High-dimensional bayesian optimization via random projection of manifold subspaces.
\newblock In Bifet, A., Daniu{\v{s}}is, P., Davis, J., Krilavi{\v{c}}ius, T., Kull, M., Ntoutsi, E., Puolam{\"a}ki, K., and {\v{Z}}liobait{\.{e}}, I., editors, {\em Machine Learning and Knowledge Discovery in Databases. Research Track and Demo Track}, pages 288--305, Cham. Springer Nature Switzerland.

\bibitem[Nguyen et~al., 2021]{nguyen2021optimal}
Nguyen, V., Le, T., Yamada, M., and Osborne, M.~A. (2021).
\newblock Optimal transport kernels for sequential and parallel neural architecture search.
\newblock In {\em International Conference on Machine Learning}, pages 8084--8095. PMLR.

\bibitem[Nguyen and Osborne, 2020]{nguyen2020knowing}
Nguyen, V. and Osborne, M.~A. (2020).
\newblock Knowing the what but not the where in {B}ayesian optimization.
\newblock In {\em International Conference on Machine Learning}, pages 7317--7326. PMLR.

\bibitem[Nocedal and Wright, 1999]{nocedal1999numerical}
Nocedal, J. and Wright, S.~J. (1999).
\newblock {\em Numerical optimization}.
\newblock Springer.

\bibitem[Parker-Holder et~al., 2022]{parker2022automated}
Parker-Holder, J., Rajan, R., Song, X., Biedenkapp, A., Miao, Y., Eimer, T., Zhang, B., Nguyen, V., Calandra, R., Faust, A., et~al. (2022).
\newblock Automated reinforcement learning (autorl): A survey and open problems.
\newblock {\em Journal of Artificial Intelligence Research}, 74:517--568.

\bibitem[Rana et~al., 2017]{rana2017high}
Rana, S., Li, C., Gupta, S., Nguyen, V., and Venkatesh, S. (2017).
\newblock High dimensional {B}ayesian optimization with elastic {G}aussian process.
\newblock In {\em International conference on machine learning}, pages 2883--2891. PMLR.

\bibitem[Rolland et~al., 2018]{rolland2018high}
Rolland, P., Scarlett, J., Bogunovic, I., and Cevher, V. (2018).
\newblock High-dimensional {B}ayesian optimization via additive models with overlapping groups.
\newblock In {\em International conference on artificial intelligence and statistics}, pages 298--307. PMLR.

\bibitem[Santoni et~al., 2024]{santoni2024comparison}
Santoni, M.~L., Raponi, E., Leone, R.~D., and Doerr, C. (2024).
\newblock Comparison of high-dimensional bayesian optimization algorithms on bbob.
\newblock {\em ACM Transactions on Evolutionary Learning}, 4(3):1--33.

\bibitem[Shen and Kingsford, 2023]{shen2021computationally}
Shen, Y. and Kingsford, C. (2023).
\newblock Computationally efficient high-dimensional bayesian optimization via variable selection.
\newblock In {\em International Conference on Automated Machine Learning}, pages 15--1. PMLR.

\bibitem[Snoek et~al., 2012]{snoek2012practical}
Snoek, J., Larochelle, H., and Adams, R.~P. (2012).
\newblock Practical {B}ayesian optimization of machine learning algorithms.
\newblock {\em Advances in neural information processing systems}, 25.

\bibitem[Song et~al., 2022]{song2022monte}
Song, L., Xue, K., Huang, X., and Qian, C. (2022).
\newblock Monte carlo tree search based variable selection for high dimensional {B}ayesian optimization.
\newblock {\em Advances in Neural Information Processing Systems}, 35:28488--28501.

\bibitem[Srinivas et~al., 2009]{Srinivas2009InformationTheoreticRB}
Srinivas, N., Krause, A., Kakade, S.~M., and Seeger, M.~W. (2009).
\newblock Information-theoretic regret bounds for {G}aussian process optimization in the bandit setting.
\newblock {\em IEEE Transactions on Information Theory}, 58:3250--3265.

\bibitem[Tibshirani, 1996]{tibshirani1996regression}
Tibshirani, R. (1996).
\newblock Regression shrinkage and selection via the lasso.
\newblock {\em Journal of the Royal Statistical Society Series B: Statistical Methodology}, 58(1):267--288.

\bibitem[Todorov et~al., 2012]{todorov2012mujoco}
Todorov, E., Erez, T., and Tassa, Y. (2012).
\newblock Mujoco: A physics engine for model-based control.
\newblock In {\em 2012 IEEE/RSJ international conference on intelligent robots and systems}, pages 5026--5033. IEEE.

\bibitem[Tran-The et~al., 2020a]{NEURIPS2020_bb073f28}
Tran-The, H., Gupta, S., Rana, S., Ha, H., and Venkatesh, S. (2020a).
\newblock Sub-linear regret bounds for bayesian optimisation in unknown search spaces.
\newblock In Larochelle, H., Ranzato, M., Hadsell, R., Balcan, M., and Lin, H., editors, {\em Advances in Neural Information Processing Systems}, volume~33, pages 16271--16281. Curran Associates, Inc.

\bibitem[Tran-The et~al., 2022a]{tran-the2}
Tran-The, H., Gupta, S., Rana, S., Truong, T., Tran-Thanh, L., and Venkatesh, S. (2022a).
\newblock Expected improvement for contextual bandits.
\newblock In Koyejo, S., Mohamed, S., Agarwal, A., Belgrave, D., Cho, K., and Oh, A., editors, {\em Advances in Neural Information Processing Systems}, volume~35, pages 22725--22738. Curran Associates, Inc.

\bibitem[Tran-The et~al., 2020b]{tran2019trading}
Tran-The, H., Gupta, S., Rana, S., and Venkatesh, S. (2020b).
\newblock Trading convergence rate with computational budget in high dimensional bayesian optimization.
\newblock {\em Proceedings of the AAAI Conference on Artificial Intelligence}, 34(03):2425–2432.

\bibitem[Tran-The et~al., 2022b]{tran-the1}
Tran-The, H., Gupta, S., Rana, S., and Venkatesh, S. (2022b).
\newblock Regret bounds for expected improvement algorithms in gaussian process bandit optimization.
\newblock In Camps-Valls, G., Ruiz, F. J.~R., and Valera, I., editors, {\em Proceedings of The 25th International Conference on Artificial Intelligence and Statistics}, volume 151 of {\em Proceedings of Machine Learning Research}, pages 8715--8737. PMLR.

\bibitem[Wan et~al., 2021]{wan2021think}
Wan, X., Nguyen, V., Ha, H., Ru, B., Lu, C., and Osborne, M.~A. (2021).
\newblock Think global and act local: {B}ayesian optimisation over high-dimensional categorical and mixed search spaces.
\newblock In {\em International Conference on Machine Learning}, pages 10663--10674. PMLR.

\bibitem[Wang, 2005]{wang2005volumes}
Wang, X. (2005).
\newblock Volumes of generalized unit balls.
\newblock {\em Mathematics Magazine}, 78(5):390--395.

\bibitem[Wang et~al., 2018]{wang2018batched}
Wang, Z., Gehring, C., Kohli, P., and Jegelka, S. (2018).
\newblock Batched large-scale {B}ayesian optimization in high-dimensional spaces.
\newblock In {\em International Conference on Artificial Intelligence and Statistics}, pages 745--754. PMLR.

\bibitem[Wang et~al., 2016]{wang2016bayesian}
Wang, Z., Hutter, F., Zoghi, M., Matheson, D., and De~Feitas, N. (2016).
\newblock {B}ayesian optimization in a billion dimensions via random embeddings.
\newblock {\em Journal of Artificial Intelligence Research}, 55:361--387.

\bibitem[Williams and Rasmussen, 2006]{williams2006gaussian}
Williams, C.~K. and Rasmussen, C.~E. (2006).
\newblock {\em {G}aussian processes for machine learning}, volume~2.
\newblock MIT press Cambridge, MA.

\bibitem[W{\"u}thrich et~al., 2021]{wuthrich2021regret}
W{\"u}thrich, M., Sch{\"o}lkopf, B., and Krause, A. (2021).
\newblock Regret bounds for gaussian-process optimization in large domains.
\newblock {\em Advances in Neural Information Processing Systems}, 34:7385--7396.

\bibitem[Xu and Zhe, 2024]{xu2024standard}
Xu, Z. and Zhe, S. (2024).
\newblock Standard gaussian process is all you need for high-dimensional bayesian optimization.
\newblock {\em arXiv preprint arXiv:2402.02746}.

\bibitem[Zheng et~al., 2003]{zheng2003statistical}
Zheng, A., Jordan, M., Liblit, B., and Aiken, A. (2003).
\newblock Statistical debugging of sampled programs.
\newblock {\em Advances in neural information processing systems}, 16.

\bibitem[Ziomek and Ammar, 2023]{ziomek2023random}
Ziomek, J.~K. and Ammar, H.~B. (2023).
\newblock Are random decompositions all we need in high dimensional {B}ayesian optimisation?
\newblock In {\em International Conference on Machine Learning}, pages 43347--43368. PMLR.

\end{thebibliography}

\onecolumn
\aistatstitle{Supplementary Materials}

\section{Related work}

Expanding Bayesian Optimization (BO) to tackle problems with high dimensions is difficult because of the curse of dimensionality and the associated computational costs. When the dimensionality increases, the search space grows exponentially, necessitating more samples and evaluations to locate a satisfactory solution, which can be costly. Additionally, optimizing the acquisition function can be extremely time-consuming. Several methods have been proposed to address high-dimensional BO, each with its own set of assumptions.

\textbf{Embedding-based method.}
A popular approach is to rely on low-dimensional structure, with several methods utilizing random projections. REMBO \citep{wang2016bayesian} uses a random projection to project low-dimensional points up to the original space. ALEBO \citep{letham2020re} introduces several refinements to REMBO and demonstrates improved performance across a large number of problems. Alternatively, the embedding can be learned jointly with the model, including both linear and non-linear embeddings. Finally, Hashing-enhanced Subspace BO (HeSBO) \citep{nayebi2019framework} relies on hashing and sketching to reduce surrogate modeling and acquisition function optimization to a low-dimensional space.

\textbf{Decomposition-based method.}
Assuming that the function can be expressed as the sum of low-dimensional functions residing in separate subspaces, \cite{kandasamy2015high} introduced the Add-GP-UCB algorithm to optimize these low-dimensional functions individually. This approach was later extended to handle overlapping subspaces.\cite{wang2018batched} proposed ensemble BO, which employs an ensemble of additive GP models to enhance scalability. \cite{han2021high} constrained the dependency graphs of the decomposition to tree structures, aiming to facilitate the learning and optimization of the decomposition. However, in many cases, the decomposition is unknown and challenging to learn. \cite{ziomek2023random} demonstrate that a random tree-based decomposition sampler offers promising theoretical guarantees by effectively balancing maximal information gain and functional mismatch between the actual black-box function and its surrogate provided by the decomposition. 

\textbf{Variable Selection based method.}
Based on the same assumption as embedding, variable selection methods iteratively select a subset of variables to build a low-dimensional subspace and optimize through BO. The selected variables can be viewed as important variables that are valuable for exploitation, or having high uncertainty that are valuable for exploration. A classic method is Dropout \citep{li2018high} which randomly selects $d$ variables in each iteration without regard to the variable's importance. MS-UCB \citep{tran2019trading} selects a fixed $d$ variables to create a low-dimensional search subspace, which can balance between computational cost and convergence rate. \cite{Chen2012JointOA} proposed a distinct procedure for identifying valid variables prior to applying Bayesian Optimization. This variable selection method depends on a particular type of structured data and does not assign a ranking to the variables based on their importance. Moreover, these methods have not yet utilized the obtained observations to automatically select variables in a reasonable way. VS-BO \citep{shen2021computationally} selects variables based on their larger estimated gradients and utilizes CMA-ES \citep{hansen2016cma} to determine the values of the unselected variables.

\cite{song2022monte} propose a variable selection method MCTS-VS based on Monte Carlo tree search (MCTS), to iteratively select and optimize a subset of variables. For the MCTS-VS method, a set of variables is sequentially divided into subsets, which can result in the important variables being divided into two subsets at different leaf nodes. This occurs when, in each iteration, the selected node is incorrectly split, causing the important set of variables to be present in both child leaf nodes. This is particularly noticeable in the early iterations when the sampled data is limited in size.

SAASBO \cite{eriksson2021high} uses sparse priors on the GP inverse lengthscales which seems particularly valuable if the active subspace is axis-aligned. However, if the number of dimensions of the problem exceeds the total evaluation budget, SAASBO will only perform on a limited number of selected variables, which is fewer than the number of effective dimensions. Moreover, SAASBO still optimizes high dimensional acquisition function, and also due to its high computational cost of inference, it is very time-consuming. 

\section{Preliminaries}

The squared exponential (SE) kernel and Matern kernel are defined as
\begin{equation}
k_{SE}^\psi(\mathbf{x}, \mathbf{x}')=\sigma_{k}^2 \exp \left\{-\frac{1}{2} \sum \rho_i\left(x_i-x'_i\right)^2\right\}
\label{kernel1}
\end{equation}
\begin{equation}
k_{Matern}^\psi(\mathbf{x}, \mathbf{x}') = \sigma^2_{k} \frac{2^{1-\nu}}{\Gamma(\nu)}\left(\sqrt{\sum \rho_i\left(x_i-x'_i\right)^2} \sqrt{2 \nu}\right)^\nu B_\nu\left( \sqrt{\sum \rho_i\left(x_i-x'_i\right)^2} \sqrt{2 \nu}\right)
\label{kernel2}
\end{equation}
where $\sigma_{k}$ and $\rho_i$ for $i=1, \ldots, D$ are hyperparameters and we use $\psi$ to collectively denote all the hyperparameters, i.e. $\psi=\left\{\rho_{1: D}, \sigma_k^2\right\}$.

\section{Notations}

\begin{table}[h!]
\centering
\resizebox{1.0\textwidth}{!}{
\begin{tabular}{cl}
\toprule 
\textbf{Variable} & \textbf{Definition} \tabularnewline
\midrule 
$\mathbf{x} = \{x_1, \dots, x_D\}, y$ & All variables in the input space and the noisy observation of the objective function $f$. \tabularnewline

$D$ & The number of variables in the original search space. \tabularnewline

$[D]$ & Set of indices $\left\{1, 2, \dots, D\right\}$, representing all dimensions.  \tabularnewline
$\mathbb{D}_t$ & The dataset at iteration $t$, containing all past samples and their observed values. \tabularnewline

$\rho_1, \dots, \rho_D$ & Inverse squared length scales of the kernel function in the Gaussian Process. \tabularnewline

$\mathbf{x}^{t,+}$ & The best sampled point after $t$ iterations $\mathbf{x}^{t,+} = \arg \max_{\mathbf{x} \in \mathbb{D}_t} f(\mathbf{x}) $.  \tabularnewline

$I_t$ & Set of indices corresponding to important variables determined at iteration $t$.  \tabularnewline

$\mathbf{x}_{I_t}$ & Subset of important variables determined at iteration $t$. \tabularnewline

$d_t$ & Number of important variables at iteration $t$. \tabularnewline

$\mathbf{x}_{[D] \setminus I_t}$ & Subset of unimportant variables at iteration $t$.  \tabularnewline

$\mathcal{X}_t = [\mathcal{X}_{I_t}, \mathcal{X}_{[D] \setminus I_t}]$ & The search space at iteration $t$, where $\mathbf{x}_{I_t} \in \mathcal{X}_{I_t}$ and $\mathbf{x}_{[D] \setminus I_t} \in \mathcal{X}_{[D] \setminus I_t}$. \tabularnewline

$\sigma^2_k$ & Variance or amplitude parameter of the kernel in GP. \tabularnewline

$\sigma^2$ & Noise variance $y \sim \mathcal{N}(f,\sigma)$.  \tabularnewline

$\lambda$ & The rate parameter of  $\mathcal{EXP}(\lambda)$. \tabularnewline

$d_e$ & Number of effective (valid) variables. \tabularnewline

$M_t$ & Number of random unimportant components. \tabularnewline

$\beta_t$ & Constant to balance exploration and exploitation at iteration $t$. \tabularnewline

$\mu_{t+1}(x)$ &  Mean of the posterior distribution over $f(x)$ at iteration $t$. \tabularnewline

$\sigma_{t+1}(x)$ &  Variance of the posterior distribution over $f(x)$ at iteration $t$. \tabularnewline

$\gamma_t$ & Maximum information gain at iteration $t$. \tabularnewline

\bottomrule
\end{tabular}
}
\vspace{1.0em}
\caption{Notations used throughout the paper. We use column vectors in all notations.}
\label{table:notation}
\end{table}

\section{Theoretical Proof}
\label{proofB}
\begin{theorem}
Let $f$ sample from GP with the SE kernel or Matern kernel. Given L, for any $\mathbf{x}\in \left[0,1\right]^D$, we have that if the kernel is
\begin{itemize}
    \item the SE kernel then $
P\left(\left|\frac{\partial f}{\partial x_i}\right| \leq L \sqrt{\rho_i}\right) \geq  1 - e^{-\frac{L^2}{2 \sigma_k^2}}
$
\item the Mate\'rn kernel ($\nu = p + \frac{1}{2} > 1$) then
$
P\left(\left|\frac{\partial f}{\partial x_i}\right| \leq L \sqrt{\rho_i}\right) \geq  1 - e^{-\frac{L^2}{2 \sigma_k^2} \frac{2p-1}{2p+1}}
$
\end{itemize} 
\end{theorem}



\begin{proof}
\textbf{In the SE kernel case: }
    For every $\mathbf{x}=\left(x_1, x_2, \ldots, x_D\right) \in \mathbb{R}^D, \epsilon \in \mathbb{R}$, let $x^{\prime}=\left(x_1, \ldots, x_i+\epsilon, \ldots, x_D\right)$, we have that
$$
\left(\begin{array}{l}
f(\mathbf{x}) \\
f\left(\mathbf{x}^{\prime}\right)
\end{array}\right) \sim \mathcal{N}\left(\left(\begin{array}{l}
0 \\
0
\end{array}\right),\left[\begin{array}{ll}
k(\mathbf{x}, \mathbf{x}) & k\left(\mathbf{x}, \mathbf{x}^{\prime}\right) \\
k\left(\mathbf{x}, \mathbf{x}^{\prime}\right) & k\left(\mathbf{x}^{\prime}, \mathbf{x}^{\prime}\right)
\end{array}\right]\right)
$$
So $f\left(\mathbf{x}^{\prime}\right)-f(\mathbf{x}) \sim \mathcal{N}(0, \mathbb{V}_{SE})$ where
\begin{equation}
    \begin{aligned}
\mathbb{V}_{SE} & =k(\mathbf{x}, \mathbf{x})+k\left(\mathbf{x}^{\prime}, \mathbf{x}^{\prime}\right)-2 k\left(\mathbf{x}, \mathbf{x}^{\prime}\right) \\
& =2 \sigma_k^2\left(1-e^{\frac{-\rho_i \epsilon^2}{2}}\right).
\end{aligned}
\label{6}
\end{equation}

Then we have that
$$
\frac{f\left(\mathbf{x}^{\prime}\right)-f(\mathbf{x})}{\epsilon} \sim \mathcal{N}\Bigl(0, \frac{2 \sigma_k^2\bigl(1-e^{\frac{-\rho_i \epsilon^2}{2}}\bigr)}{\epsilon^2} \Bigr).
$$
Let $\epsilon \rightarrow 0$, we obtain $\frac{\partial f}{\partial x_i}(\mathbf{x}) \sim \mathcal{N}\left(0, \sigma_k^2 \rho_i\right)$.
This implies that, given $L$, for any $\mathbf{x}$, we have
$$
P\left(\left|\frac{\partial f}{\partial x_i}\right| \leq L \sqrt{\rho_i}\right) \geq  1 - e^{-\frac{L^2}{2 \sigma_k^2}}.
$$



\textbf{In the Matern kernel case: }
For every $\mathbf{x}=\left(x_1, x_2, \ldots, x_D\right) \in \mathbb{R}^D, \epsilon \in \mathbb{R}$, let $x^{\prime}=\left(x_1, \ldots, x_i+\epsilon, \ldots, x_D\right)$, we have that
$$
\left(\begin{array}{l}
f(\mathbf{x}) \\
f\left(\mathbf{x}^{\prime}\right)
\end{array}\right) \sim \mathcal{N}\left(\left(\begin{array}{l}
0 \\
0
\end{array}\right),\left[\begin{array}{ll}
k(\mathbf{x}, \mathbf{x}) & k\left(\mathbf{x}, \mathbf{x}^{\prime}\right) \\
k\left(\mathbf{x}, \mathbf{x}^{\prime}\right) & k\left(\mathbf{x}^{\prime}, \mathbf{x}^{\prime}\right)
\end{array}\right]\right)
$$
So $f\left(\mathbf{x}^{\prime}\right)-f(\mathbf{x}) \sim \mathcal{N}(0, \mathbb{V}_{SE})$ where
\begin{equation}
    \begin{aligned}
\mathbb{V}_{Matern} & =k(\mathbf{x}, \mathbf{x})+k\left(\mathbf{x}^{\prime}, \mathbf{x}^{\prime}\right)-2 k\left(\mathbf{x}, \mathbf{x}^{\prime}\right) \\
& =2 \sigma_k^2\left(1 - C_{p + 1/2}(\epsilon)\right)
\end{aligned}
\end{equation}
where 
$$
    \begin{aligned}
C_{p+1 / 2}(\epsilon) &= e ^ {-\sqrt{2 p+1} \sqrt{\rho_i} \left|\epsilon \right| } \frac{p !}{(2 p) !} \sum_{i=0}^p \frac{(p+i) !}{i !(p-i) !}\left(2 \sqrt{2 p+1} \sqrt{\rho_i} \left|\epsilon \right| \right)^{p-i} \\
& = 1 - \frac{2p + 1}{4p - 2} \rho_i  \epsilon^2 + \mathcal{O}(\epsilon ^ 3).
\end{aligned}
$$
Then we have that
$$
\frac{f\left(\mathbf{x}^{\prime}\right)-f(\mathbf{x})}{\epsilon} \sim \mathcal{N}\left(0, \frac{\mathbb{V}_{M}}{\epsilon^2}\right)
$$
from $\lim_{\epsilon \rightarrow 0} \frac{\mathbb{V}_{M}}{\epsilon^2} = \frac{2p+1}{2p-1} \sigma_k^2 \rho_i$ , thus $\frac{\partial f}{\partial x_i} \sim \mathcal{N} \left(0, \frac{2p+1}{2p-1} \sigma_k^2 \rho _i\right)$. It implies that
$$
P\left(\left|\frac{\partial f}{\partial x_i}\right| \leq L \sqrt{\rho_i}\right) \geq  1 - e^{-\frac{L^2}{2 \sigma_k^2} \frac{2p-1}{2p+1}}.
$$

\end{proof}
\section{Proof of Theorem 1}
\label{maintheorem}

To derive a cumulative regret $R_t = \sum_{t=1}^{T} r_t$, we will seek to bound $r_t = f(\mathbf{x}^*)- f(\mathbf{x}_t)$ for any $t$. At each iteration $t$, we denote by $v=\mathbf{x}_{I_t}, z = \mathbf{x}_{[D] \setminus I_t}, x = [v ,z]$ and $\mathcal S_t^0 = \left[v^*, z_i \right]_{z_i \in Z_t}$. Let $\left[v^*, z_t^* \right] = \text{argmax}_{\mathbf{x} \in \mathcal S_t^0}\{f(\mathbf{x})\}$ and $f^{\text{\text{max}}}_{\mathcal S_t^0} = f\left[v^*, z_t^* \right]$.
To obtain a bound on $r_t$, we write it as
\begin{eqnarray}
r_t & = & f(\mathbf{x}^*)- f(\mathbf{x}_t) \\
& = & \underbrace{f(\mathbf{x}^*) - f^{\text{max}}_{\mathcal S_t^0}}_\text{Term 1} +  \underbrace{f^{\text{max}}_{\mathcal S_t^0}}_\text{Term 2} - \underbrace{f(\mathbf{x}_t)}_\text{Term 3}
\end{eqnarray}
Term 1 will be bounded from the result of Lemma \ref{lem:5.1}, Term 2 will be bounded from the result of Lemma \ref{lem:5.3} and Term 3 will bounded from the result of Lemma \ref{lem:5.4}.

Similar to Lemma 5.5 of \cite{Srinivas2009InformationTheoreticRB}, we have the following bound on the actual function observations.
\begin{lemma}[Bounding Term 3]
Pick a $\delta \in (0,1)$ and set $\beta_t^0 = 2\log(\frac{\pi^2t^2}{6\delta})$. Then we have
\begin{eqnarray}
f(\mathbf{x}_t) \ge \mu_{t-1}(\mathbf{x}_t) - \sqrt{\beta_t^0} \sigma_{t-1}(\mathbf{x}_t)
\end{eqnarray}
holds with probability $\ge 1 - \delta$.
\label{lem:5.4}
\end{lemma}
\begin{proof}
Given $\beta_t^{0}$, we have $\mathbb{P}[|f(\mathbf{x}_t) - \mu_{t-1}(\mathbf{x}_t)| > \sqrt{\beta_t^0}] \le e^{\beta_t^0/2}$. Since $e^{\beta_t^0/2} = \frac{6\delta}{\pi^2t^2}$. Using the union bound for $t\ \in \mathbb{N}$, we have $\mathbb{P}[|f(\mathbf{x}_t) - \mu_{t-1}(\mathbf{x}_t)| \le \sqrt{\beta_t^0}] \ge 1 - \frac{\pi^2\delta}{6} > 1- \delta$. The statement holds.
\end{proof}
Now let us choose a discretization $\mathcal F_t \subset [\textcolor{black}{\mathcal{X}_{I_t}}, z^*_t]$ of size $= (4\sqrt{\rho^*}bd_t\sqrt{\log(\frac{2Da}{\delta})}t^2)^{d_t}$.

\begin{lemma}[Bounding Term 2]
Pick a $\delta \in (0,1)$ and set $\beta_{t}^1 = 2\log\left(\frac{\pi^2t^2}{3\delta} \right) + 2d\log \left(4\sqrt{\rho^*}bd_t\sqrt{\log(\frac{4Da}{\delta})}t^2 \right)$ where $\rho^* = \max_{i\in [D]} \rho_i$. Then, under Assumption \ref{as:1} there exists a $\mathbf{x}^\prime \in [\textcolor{black}{\mathcal{X}_{I_t}}, z^*_t]$ such that
\begin{eqnarray}
f^{\text{max}}_{\mathcal S_t^0} \le  \mu_{t-1}(\mathbf{x}^\prime) + \sqrt{\beta_t^1}\sigma_{t-1}(\mathbf{x}^\prime)  + \frac{1}{t^2}
\end{eqnarray}
holds with probability $\ge 1 - \delta$.
\label{lem:5.3}
\end{lemma}
\begin{proof}
We use the idea of proof of Lemma 5.7 in \cite{Srinivas2009InformationTheoreticRB} with several modifications to adapt our method with low-dimensional subspaces. 

For every $v, v' \in [0,1]^{d_t}$, we have
\begin{eqnarray*}
|| v - v' ||_1& = & \sum_{i =1}^{d_t} |[(v-v')]_i|  \\
&\leq& d_t.
\end{eqnarray*}


Now, on subspace $[\textcolor{black}{\mathcal{X}_{I_t}}, z^*_t]$, we construct a discretization $F_t$ of size $(\tau_t)^d$ dense enough such that for any $x \in F_t$ $$||\mathbf{x} - [\mathbf{x}]_t]||_1 \le \frac{d_t}{\tau_t}$$
where $[\mathbf{x}]_t$ denotes the closest point in $F_t$ to $\mathbf{x}$. By Assumption \ref{as:1} and the union bound, we have that for $\forall \mathbf{x} \in \mathcal X$,
$$|f(\mathbf{x}) - f(\mathbf{x}^\prime)| \le L\sum_{i=1}^D \sqrt{\rho_i}|[\mathbf{x} -\mathbf{x}^\prime]_i|$$
with probability greater than $1 - Dae^{-L^2/b^2}$. 
Thus,
\begin{eqnarray}
|f(\mathbf{x}) - f(\mathbf{x}^\prime)| \le b\sqrt{\log(\frac{2Da}{\delta})}\sum_{i=1}^D \sqrt{\rho_i}|[\mathbf{x} -\mathbf{x}^\prime]_i|.
\label{eq:100}
\end{eqnarray}

With probability greater than $1 - \delta/2$, for every $x \in \mathcal S(v, z^*_t)$ we have
\begin{eqnarray*}
|f(\mathbf{x}) - f([\mathbf{x}]_t)| & \le &  b\sqrt{\log(\frac{2Da}{\delta})}\sum_{i=1}^{d_t} \sqrt{\rho_i}|[v -[v]_t]_i| \\
& \le &  b\sqrt{\log(\frac{2Da}{\delta})}\sqrt{\rho^*}\frac{d_t}{\tau_t} \\
& \leq & \frac{\sqrt{\rho^*}bd_t\sqrt{\log(\frac{2Da}{\delta})}}{\tau_t}.
\end{eqnarray*}
Where $\rho^* = \max_{i\in [D]} \rho_i$. Let $\tau_t = 4\sqrt{\rho^*}bd_t\sqrt{\log(\frac{2Da}{\delta})}t^2$. Thus, $|F_t| = (4\sqrt{\rho^*}bd_t\sqrt{\log(\frac{2Da}{\delta})}t^2)^{d_t}$. We obtain 
\begin{eqnarray}
|f(\mathbf{x}) - f([\mathbf{x}]_t)| \le \frac{1}{t^2}
\label{eq:102}
\end{eqnarray}
with probability $1- \delta/2$ for any $x \in F_t$.

Similar to Lemma 5.6 of \cite{Srinivas2009InformationTheoreticRB}, if we set $\beta_{t}^1 = 2\log(|F_t|\frac{\pi^2t^2}{6\delta}) = 2\log(\frac{\pi^2t^2}{3\delta}) +  2d\log(4\sqrt{\rho^*}b\tilde{d}_t\sqrt{\log(\frac{4Da}{\delta})}t^2)$, we have with probability $1 -\delta/2$, we have
\begin{eqnarray}
f(\mathbf{x}) \le \mu_{t-1}(\mathbf{x}) + \sqrt{\beta_{t}^1}\sigma_{t-1}(\mathbf{x})
\label{eq:103} 
\end{eqnarray}
for any $x \in F_t$ and any $t \ge 1$. Thus, combining Eq. (\ref{eq:102}) and Eq. (\ref{eq:103}), if we let $\mathbf{x}^\prime = [v^*, z^*_t]_t$ which is the closest point in $F_t$ to $[v^*, z^*_t]$, we have
\begin{eqnarray*}
f^{\text{max}}_{\mathcal S_t^0} & = & f([y^*, z^*_t]) \\
&\le& f([v^*, z^*_t]_t) + \frac{1}{t^2} \\
&\le& \mu_{t-1}(\mathbf{x}^\prime) + \sqrt{\beta_{t}^{1}}\sigma_{t-1}(\mathbf{x}^\prime) + \frac{1}{t^2}
\end{eqnarray*}
with probability $1 - \delta$. Note that since $\mathbf{x}^\prime \in F_t$, $\mathbf{x}^\prime \in \mathcal S(v, z^*_t)$. Thus, the statement holds.
\end{proof}
\begin{lemma}[Bounding Term 1]
Pick a $\delta \in (0,1)$ and set $n = 2b\sqrt{\log(\frac{2Da}{\delta})}(\Gamma(D-d +1))^{\frac{1}{D-d_t}}$, where $\Gamma(D-d +1) = (D-d+1)!$. With probability at least $1- \delta$, it holds that
\begin{align}
f(\mathbf{x}^*)- f_{\mathcal{S}_t^0}^{\text{max}} \le n \sqrt{\sum_{i\in [D ] \setminus I_t} \rho_i} \left(\frac{1}{M_t}\log\left(\frac{2}{\delta}\right) \right)^{\frac{1}{D-d_t}}.
\end{align}
\label{lem:5.1}
\end{lemma}
\begin{proof}
Given any $x \in \mathcal S_t^0$. Without loss of generality, we assume that $\mathbf{x} = [y^*, z]$, where $z \in Z_t$. By Assumption \ref{as:1} and the union bound, we have that for $\forall \mathbf{x}$,
$$f(\mathbf{x}^*)- f_{\mathcal{S}_t^0}^{\text{max}} \le |f(\mathbf{x}^*) - f(\mathbf{x})| \le L\sum_{i=1}^D\sqrt{\rho_i}|[\mathbf{x}^* -\mathbf{x}]_i| \leq L\sum_{i \in I_t}\sqrt{\rho_i}|[z^* -z]_i|$$
with probability greater than $1 - Dae^{-L^2/b^2}$.
Set $Dae^{-L^2/b^2} = \delta/2$. Thus,
\begin{eqnarray}
|f(\mathbf{x}^*) - f(\mathbf{x})| \le b\sqrt{\log(\frac{2Da}{\delta})}\sqrt{\sum_{i \in I_t} \rho_i}||z^* -z||_2
\label{eq:100}
\end{eqnarray}
with probability greater than $1 - \delta/2$.

We denote the volume of space $\mathcal Z$ be $\text{Vol}_0$. Since $z  \in [0, 1]^{D-d_t}$, $\text{Vol}_0 = 1$. For any $\theta > 0$, let $V_1 = \{z \in \mathcal Z = [0, 1]^{D-d_t}| ||z-z^*||_1 \le \theta\}$. In the case where $V_1$ is within $\mathcal Z$, the volume $\text{Vol}(V_1)$ is the volume of a $(D-d)$-dimensional ball of radius $\theta$ with center $z^*$ and thus, $\text{Vol}(V_1) = \frac{(\theta)^{D-d_t}}{\Gamma(D-d +1)}$ where $\Gamma$ is Leonhard Euler's Gamma function. The function $\Gamma(k) = (k-1)!$ for any $k$ is a positive integer. Here, we used the formula in \cite{wang2005volumes} for the volume of $(D-d)$-dimensional balls in $L^1$ norms. In the worst case where $z^*$ is at the boundary of $\mathcal Z$, more precisely either $[z^*]_i = 1$ for any $1 \le i \le D-d$, or $[z^*]_i = 0$ for any $1 \le i \le D-d$, $\text{Vol}(V_1)$ is at least $\frac{(\theta)^{D-d_t}}{\Gamma(D-d +1)2^{D-d_t}}$. This is because the volume halved in each dimension. Thus, in every case, we have $\text{Vol}(V_1) \ge \frac{(\theta)^{D-d_t}}{\Gamma(D-d +1)2^{D-d_t}}$.

Thus, $\mathbb{P}[||z- z^*||_1 \le \theta] = \frac{\text{Vol}(V_1)}{\text{Vol}_0} \ge \frac{\frac{(\theta)^{D-d_t}}{\Gamma(D-d +1)}}{2^{D-d_t}} = \frac{\theta^{D-d_t}}{\Gamma(D-d +1)2^{D-d_t}}$. However, $\mathbb{P}[||z- z^*||_1 > \theta] = 1 - \mathbb{P}[||z- z^*||_1 \le \theta]$. Thus, $\mathbb{P}[||z- z^*||_1 > \theta] \le 1 - \frac{\theta^{D-d_t}}{\Gamma(D-d +1)2^{D-d_t}} \le e^{- \frac{\theta^{D-d_t}}{\Gamma(D-d +1)2^{D-d_t}}}$, where we use the inequality $1-x \le e^{-x}$ for the last step.

Therefore, we obtain
\begin{eqnarray*}
\prod_{z \in \mathcal Z_t} \mathbb{P}[||z- z^*||_2 > \theta] & \le & \prod_{z \in \mathcal Z_t} e^{- \frac{\theta^{D-d_t}}{\Gamma(D-d_t +1)2^{D-d_t}}}\\
& = & e^{-M_t\frac{\theta^{D-d_t}}{\Gamma(D-d_t +1)2^{D-d_t}}}.
\end{eqnarray*}
On the other hand, $ \mathbb{P}[||z_t^*- z^*||_2 \le \theta] = 1 - \mathbb{P}[||z_t^*- z^*||_2 > \theta]   = 1-  \mathbb{P}_{z^i_t \in \mathcal Z_t, 1 \le i \le M_t}[||z_t^1- z^*||_2 > \theta \wedge...\wedge  ||z_t^{M_t}- z^*||_2 > \theta]  = 1- \prod_{z^i_t \in \mathcal Z_t, 1 \le i \le M_t} \mathbb{P}[||z_t^i- z^*||_2 > \theta]$. This is because $z^i_{t}$ are independent due to $z^i_t$ sampled uniformly at random in $\mathcal Z$. Thus, we have
\begin{eqnarray*}
\mathbb{P}[||z_t^*- z^*||_2  \le \theta] \ge 1 - e^{-M_t\frac{\theta^{D-d_t}}{\Gamma(D-d_t +1)2^{D-d_t}}}.
\end{eqnarray*}
Now we set $e^{-M_t\frac{\theta^{D-d_t}}{\Gamma(D-d_t +1)2^{D-d_t}}}= \delta/2$ to obtain $\theta = 2(\Gamma(D-d_t +1))^{\frac{1}{D-d_t}}(\frac{1}{M_t}\log\left(\frac{2}{\delta}\right))^{\frac{1}{D-d_t}}$. Thus,
\begin{eqnarray}
||z_t^*- z^*||_2  \le 2(\Gamma(D-d_t +1))^{\frac{1}{D-d_t}}(\frac{1}{M_t}\log\left(\frac{2}{\delta}\right))^{\frac{1}{D-d_t}}
\label{eq:101}
\end{eqnarray}
with probability $1 - \delta/2$. Combining Eq. (\ref{eq:100}) and Eq. (\ref{eq:101}), we have
\begin{eqnarray}
|f(\mathbf{x}^*) - f(\mathbf{x})| \le n \sqrt{\sum_{i\in [D ] \setminus I_t} \rho_i} (\frac{1}{M_t}\log\left(\frac{2}{\delta}\right))^{\frac{1}{D-d_t}}
\end{eqnarray}
with probability $1 - \delta$. Thus, the lemma holds.
\end{proof}
Now, we combine the results from Lemmas \ref{lem:5.4}, \ref{lem:5.3} and \ref{lem:5.1} to obtain a bound on $r_t$ as stated in the following Lemma.
\begin{lemma}
Pick a $\delta \in (0,1)$ and set $\beta_t = 2\log(\frac{\pi^2t^2}{\delta}) + 2d_t\log(4\sqrt{\rho^*} bd_t\sqrt{\log(\frac{12Da}{\delta})}t^2)$ and set $n = 2b\sqrt{\log(\frac{2Da}{\delta})}(\Gamma(D-d_t +1))^{\frac{1}{D-d_t}}$, where $\Gamma(D-d_t +1) = (D-d_t+1)!$ . Then we have
\begin{eqnarray}
r_t \le 2\beta_t^{1/2}\sigma_{t-1}(\mathbf{x}_t) + \frac{1}{t^2} + n\sqrt{\sum_{i\in [D ] \setminus I_t} \rho_i}\left(\frac{1}{M_t}\log\left(\frac{6}{\delta}\right)\right)^{\frac{1}{D-d_t}}
\end{eqnarray}
holds with probability $\ge 1 - \delta$.
\label{lem:5.8}
\end{lemma}
\begin{proof}
We use $\frac{\delta}{3}$ for Lemmas \ref{lem:5.4}, \ref{lem:5.3} and \ref{lem:5.1} so that these events hold simultaneously with probability greater than $1- \delta$. Formally, by Lemma \ref{lem:5.4} using $\frac{\delta}{3}$:
\begin{eqnarray*}
f(\mathbf{x}_t) \ge \mu_{t-1}(\mathbf{x}_t) - \sqrt{2\log(\frac{\pi^2t^2}{\delta})}\sigma_{t-1}(\mathbf{x}_t)
\end{eqnarray*}
holds with probability $\ge 1 - \frac{\delta}{3}$. As a result,
\begin{eqnarray}
f(\mathbf{x}_t) & \ge &  \mu_{t-1}(\mathbf{x}_t) - \sqrt{\beta_t^0}\sigma_{t-1}(\mathbf{x}_t) \\
& > & \mu_{t-1}(\mathbf{x}_t) - \sqrt{\beta_t}\sigma_{t-1}(\mathbf{x}_t) \label{eq:50}
\end{eqnarray}
holds with probability $\ge 1 - \frac{\delta}{3}$.

By Lemma \ref{lem:5.3} using $\frac{\delta}{3}$, there exists a $\mathbf{x}^\prime \in \mathcal S(A, z^*_t)$ such that
\begin{eqnarray*}
f^{\text{max}}_{\mathcal S_t^0} \le \mu_{t-1}(\mathbf{x}^\prime) + \sqrt{\beta_t^{1}}\sigma_{t-1}(\mathbf{x}^\prime) + \frac{1}{t^2}
\end{eqnarray*}
holds with probability $\ge 1 - \frac{\delta}{3}$. As a result,
\begin{eqnarray*}
f^{\text{max}}_{\mathcal S_t^0} & \le & \mu_{t-1}(\mathbf{x}^\prime) + \sqrt{\beta_t^{1}}\sigma_{t-1}(\mathbf{x}^\prime) + \frac{1}{t^2} \\
f^{\text{max}}_{\mathcal S_t^0} & \le & \mu_{t-1}(\mathbf{x}^\prime) + \sqrt{\beta_t}\sigma_{t-1}(\mathbf{x}^\prime) + \frac{1}{t^2} \\
& = & a_t(\mathbf{x}^\prime) + \frac{1}{t^2}.
\end{eqnarray*}
Since $\mathbf{x}_t = \text{argmax}_{\mathbf{x} \in \mathcal{X}_t}a_t(\mathbf{x})$ (defined in Algorithm 1) and $\mathbf{x}^\prime \in [\textcolor{black}{\mathcal{X}_{I_t}}, z^*_t] $, we have $a_t(\mathbf{x}^\prime) \le a_t(\mathbf{x}_t)$. Thus,
\begin{eqnarray}
 f^{\text{max}}_{\mathcal S_t^0} \le a_t(\mathbf{x}_t) + \frac{1}{t^2} \label{eq:51}
\end{eqnarray}
holds with probability $\ge 1 - \frac{\delta}{3}$.

By Lemma \ref{lem:5.1} using $\frac{\delta}{3}$:
\begin{eqnarray}
f(\mathbf{x}^*)- f_{\mathcal S_t^0}^{\text{\text{max}}} \le n \sqrt{\sum_{i\in [D ] \setminus I_t} \rho_i} \left(\frac{1}{M_t}\log\left(\frac{2}{\delta}\right) \right)^{\frac{1}{D-d_t}}\label{eq:52}
\end{eqnarray}
holds with probability $\ge 1 - \frac{\delta}{3}$.

Combining Eq. (\ref{eq:50}), Eq. (\ref{eq:51}) and Eq. (\ref{eq:52}), we have
\begin{eqnarray*}
r_t & = & f(\mathbf{x}^*) - f(\mathbf{x}_t)\\
& = & \underbrace{f(\mathbf{x}^*) - f^{\text{\text{max}}}_{\mathcal S_t^0}}_\text{Term 1} +  \underbrace{f^{\text{\text{max}}}_{\mathcal S_t^0}}_\text{Term 2} - \underbrace{f(\mathbf{x}_t)}_\text{Term 3}\\
& \le &  n(\frac{1}{M_t}\log(\frac{3}{\delta}))^{\frac{1}{D-d_t}}  +  \underbrace{f^{\text{\text{max}}}_{\mathcal S_t^0}}_\text{Term 2} - \underbrace{f(\mathbf{x}_t)}_\text{Term 3}\\
& \le & n \sqrt{\sum_{i\in [D ] \setminus I_t} \rho_i} \left(\frac{1}{M_t}\log\left(\frac{2}{\delta}\right) \right)^{\frac{1}{D-d_t}} + \frac{1}{t^2} + a_{t}(\mathbf{x}_t)  - f(\mathbf{x}_t) \label{eq:2}\\
& \le & n \sqrt{\sum_{i\in [D ] \setminus I_t} \rho_i} \left(\frac{1}{M_t}\log\left(\frac{2}{\delta}\right)\right)^{\frac{1}{D-d_t}} + \frac{1}{t^2} + 2(\beta_t)^{1/2}\sigma_{t-1}(\mathbf{x}_t)
\label{eq:3}
\end{eqnarray*}
holds with probability $\ge 1 - \delta$.
\end{proof}

\begin{lemma}(Bounding $p$-series when $p \le 1$, \cite{chlebus2009approximate})
Given a $p$-series $s_n = \sum_{k =1}^{n}\frac{1}{k^p}$, where $n \in \mathbb{N}$. Then,
\begin{itemize}
  \item if $p < 0$ then $1 + \frac{n^{1-p}-1}{1-p} < s_n < \frac{(n+1)^{1-p}-1}{1-p}$
  \item if $0 \le p < 1$ then $s_n < 1 + \frac{n^{1-p}-1}{1-p}$
\end{itemize}
\label{lem:5}
\end{lemma}

\begin{proposition}
Pick a $\delta \in (0,1)$. Set $\rho^* = \max_{i\in [D]} \rho_i$ and $\beta_t = 2\log(\frac{\pi^2t^2}{\delta}) + 2\tilde{d}_t\log(4\sqrt{\rho^*}b\tilde{d}_t\sqrt{\log(\frac{12Da}{\delta})}t^2)$, where $\tilde{d}_t = \max_{1\leq k \leq t} d_t$ and $C = \text{\text{max}}_{t \in [T]} 2b\sqrt{\log(\frac{2Da}{\delta})}(\Gamma(D-d_t+1))^{\frac{1}{D-d_t} }\log\left(\frac{6}{\delta}\right)$, where $\Gamma(D-d_t +1) = (D-d_t+1)!$. Then, the cumulative regret of the proposed algorithm is bounded by

$$ R_T \le \sqrt{\beta_TC_1T\gamma_T} + C \sqrt{\left( \sum^T_{t=1} \sum_{i\in [D ] \setminus I_t} \rho_i  \right) \left(\sum^T_{t=1}\left(\frac{1}{M_t}\right)^{\frac{2}{D-d_t}}\right)} + \frac{\pi^2}{6},$$
with probability greater than $1 -\delta$, where $C_1 = 8/\log(1 + \sigma^2)$, $\gamma_T$ is the
maximum information gain about the function $f$ from any set of observations of size $T$.
\label{prop1}
\end{proposition}
\begin{proof}
    By Lemma \ref{lem:5.8}, we have 
    $$R_T = \sum_{t=1}^{T} r_t \leq \sum_{t=1}^{T}2\beta_t^{1/2}\sigma_{t-1}(\mathbf{x}_t) + \sum_{t=1}^{T} \frac{1}{t^2} +  \sum_{t=1}^Tn\sqrt{\sum_{i\in [D ] \setminus I_t} \rho_i} \left(\frac{1}{M_t}\log\left(\frac{6}{\delta}\right)\right) ^{\frac{1}{D-d_t}}
    .$$

Similar to the proof of Theorem 2 in \cite{Srinivas2009InformationTheoreticRB}, we have that $\sum_{1}^{T}2\beta_t^{1/2}\sigma_{t-1}(\mathbf{x}_t) \le \sqrt{C_1T\beta_T\gamma_T}$ and $\sum_{1}^{T} \frac{1}{t^2} < \frac{\pi^2}{6}$.

\begin{eqnarray*}
    &\sum_{t=1}^T n\sqrt{\sum_{i\in [D ] \setminus I_t} \rho_i}\left(\frac{1}{M_t}\log\left(\frac{6}{\delta}\right)\right)^{\frac{1}{D-d_t}} &\leq C  \sum_{t=1}^T \sqrt{\sum_{i\in [D ] \setminus I_t} \rho_i} \left(\frac{1}{M_t}\right)^{\frac{1}{D-d_t}}\\
    & &\leq C \sqrt{\left( \sum^T_{t=1} \sum_{i\in [D ] \setminus I_t} \rho_i \right) \left(\sum^T_{t=1}\left(\frac{1}{M_t}\right)^{\frac{2}{D-d_t}}\right)}.
\end{eqnarray*}

\end{proof}

\begin{proposition}
Pick a $\delta \in (0,1)$. Set $M_t = \lceil \sqrt[n]{t} \rceil$ and $\beta_t = 2\log\left(\frac{\pi^2t^2}{\delta}\right) + 2\tilde{d}_t\log(4\sqrt{\rho^*}b\tilde{d}_t\sqrt{\log(\frac{12Da}{\delta})}t^2)$ and $C = \max_{t \in [T]} 2b\sqrt{\log\left(\frac{2Da}{\delta}\right)}(\Gamma(D-d_t +1))^{\frac{1}{D-d_t}}\log\left(\frac{6}{\delta}\right)$, where $\Gamma(D-d_t +1) = (D-d_t+1)!$. Then, the cumulative regret of the proposed algorithm is bounded by
$$ R_T \le \sqrt{\beta_TC_1T\gamma_T} + C \sqrt{c_0} \sqrt{\left( \sum^T_{t=1} \sum_{i\in [D ] \setminus I_t} \rho_i \right)} T^{\frac{1}{2} - \frac{1}{n(D-1)}} + \frac{\pi^2}{6},$$
with probability greater than $1 -\delta$, where $C_1 = 8/\log(1 + \sigma^2)$, $\gamma_T$ is the
maximum information gain about the function $f$ from any set of observations of size $T$, $c_1 > \frac{n(D-1)}{n(D-1)-2}$.
\label{pro}
\end{proposition}
\begin{proof}
We have $\left(\frac{1}{M_t}\right)^{\frac{1}{D-d_t}} \le \left(\frac{1}{M_t}\right)^{\frac{1}{D-1}}$. On the other hand, we choose $M_t = \lceil\sqrt[n]{t} \rceil$. We have $\sum_{t=1}^{T} \left(\frac{1}{M_t}\right)^{\frac{2}{D-1}} \le \sum_{t=1}^{T} t^{-\frac{2}{n(D-1)}} \le T^{1-p} \left(\frac{1}{T}\sum_{t=1}^T (\frac{t}{T})^{-p}\right) \le c_0 T^{1-p}$, where $p = \frac{2}{n(D-1)}$, $c_0 > \int _0^1 x^{-p}dx = \frac{1}{1-p} = \frac{n(D-1)}{n(D-1)-2}$. Thus, $\sum_{t=1}^{T} \left(\frac{1}{M_t}\right)^{\frac{2}{D-1}} \le c_0 T^{1 - \frac{2}{n(D-1)}}$.

Thus, by combining Proposition \ref{prop1}, the proposition holds.
\end{proof}

\section{Additional Experiments}
\label{sec:addexperiment}

\paragraph{Compare across the parameter $d$.} We evaluated the algorithms across various $d(d=10,15,20)$ on the Levy function with $D=300, d_e=15$, as detailed in Figure \ref{fig:d}. In this benchmark, HesBO achieves the best performance when $d$ is precisely set to $d_e=15$, with LassoBO ranking second. This highlights the critical importance of accurately selecting $d$ and underscores the strength of our algorithm in automatically identifying the important dimensions.

\paragraph{Compare with EI2 and UCB2.} We also now implemented EI2 and UCB2 \cite{wuthrich2021regret} as the reviewer suggested and evaluated the performance of our proposed LassoBO and the baselines on the levy function with $D=300, d_e=15$, as depicted in the Figure \ref{fig:d}. The figure shows that LassoBO outperforms these two methods.

\begin{figure*}[h!]
\centering
\includegraphics[width=0.45\textwidth]{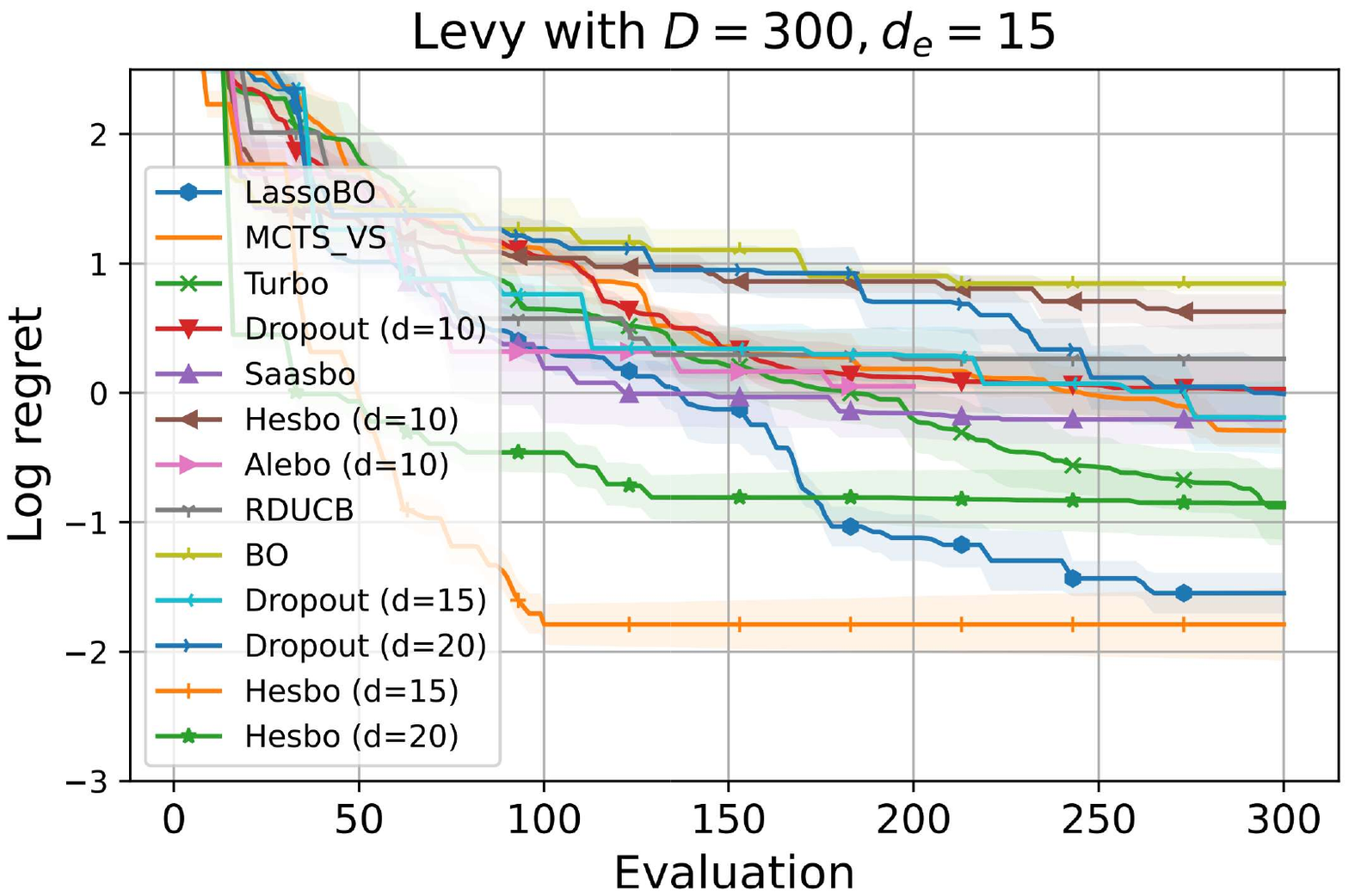}
\includegraphics[width=0.45\textwidth]{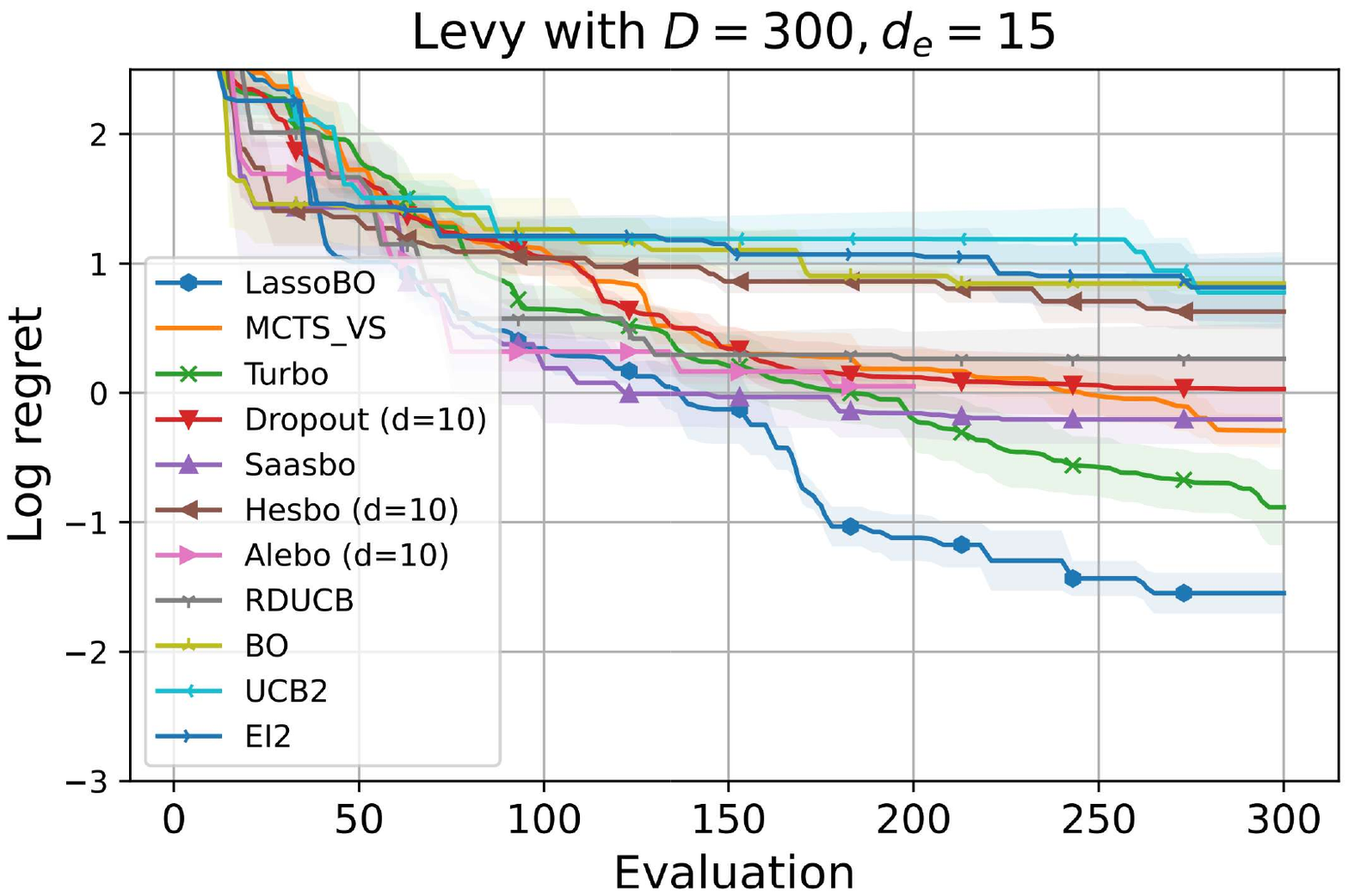}
\caption{\textbf{Left:} Comparison of different values of \( d \) (\( d = 10, 15, 20 \)) on the Levy function with \( D = 300 \) and \( d_e = 15 \). \textbf{Right:} Comparation with EI2 and UCB2 on the Levy function with \( D = 300 \) and \( d_e = 15 \).} \label{fig:d}
\end{figure*}

\paragraph{Compare the efficiency of variable selection with MCTS\_VS.} We compare their ability to select variables on the Sum Squares function ($D=300, d_e = 15$), i.e $f(\mathbf{x}) = \sum \alpha_i x_i^2$ as shown in Fig. \ref{fig:ss20300}. In the figure, the first subfigure indicates that the number of selected variables in LassoBO begins at a minimal level and progresses to the number of effective dimensions. In contrast, MCTS\_VS starts with a relatively high number of selected variables and progressively decreases to a level smaller than the effective number of dimensions. The 2nd subfigure depicts the variables selected in LassBO and the 3rd subfigure shows the variables selected in MCTS\_VS over iterations. The results show that LassoBO mainly selects effective dimensions while MCTS\_VS selects many ineffective dimensions, implying the superiority of variable selection in LassoBO over MCTS\_VS.
\begin{figure*}[h!]
{\includegraphics[width=0.32\textwidth]{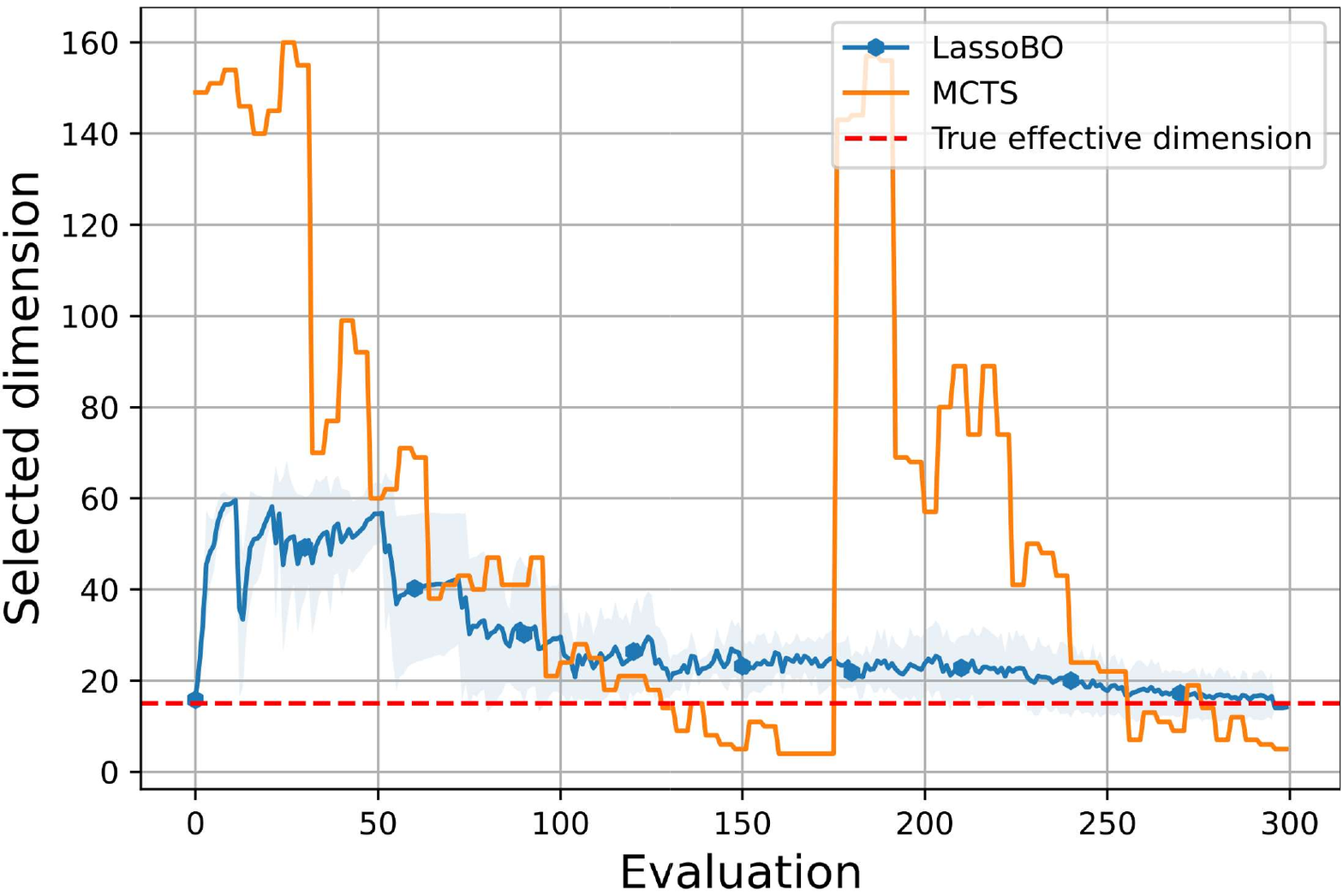}
    \label{fig:ss20300a}}
{\includegraphics[width=0.32\textwidth]{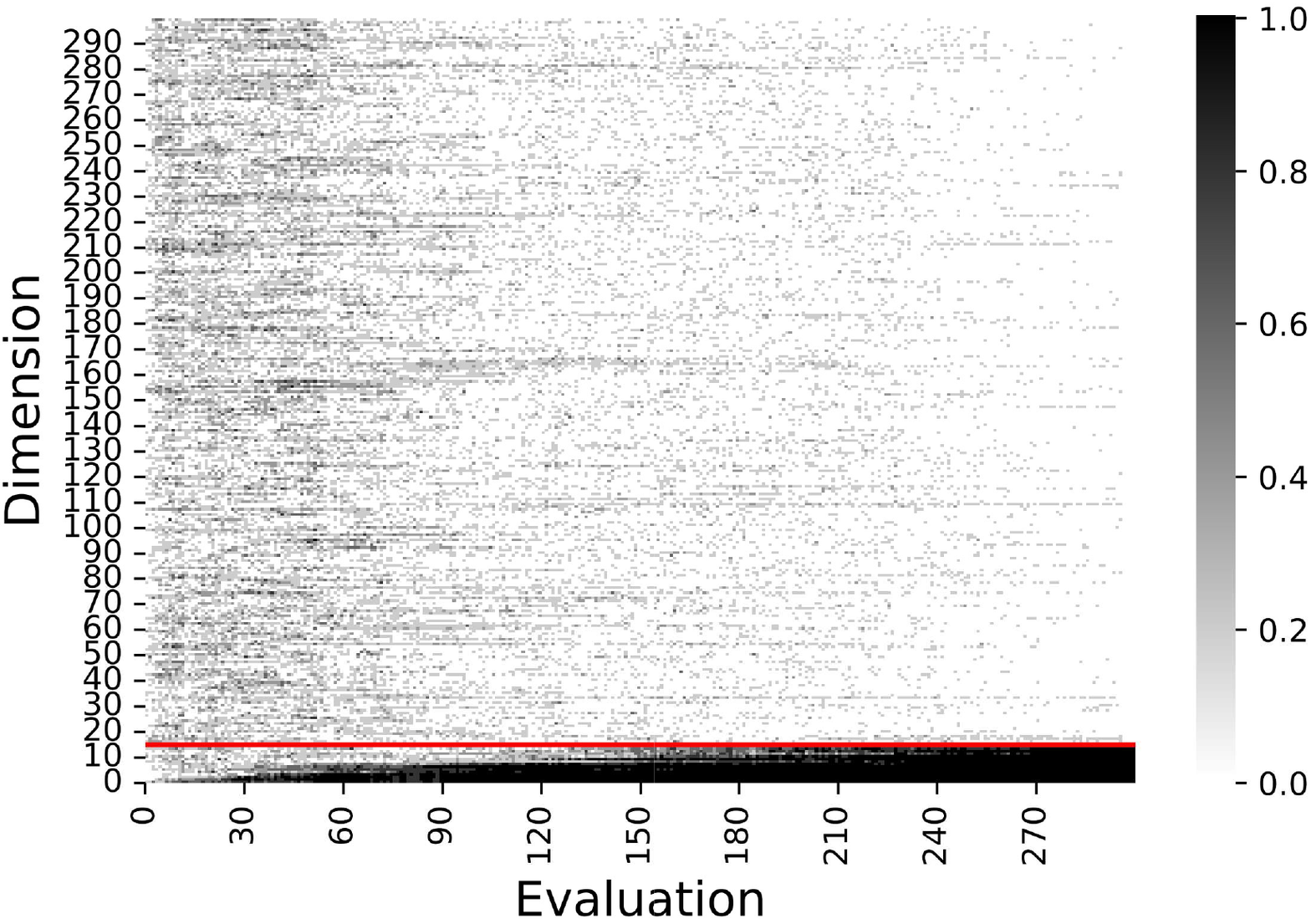}
    \label{fig:ss20300b}}
{\includegraphics[width=0.32\textwidth]{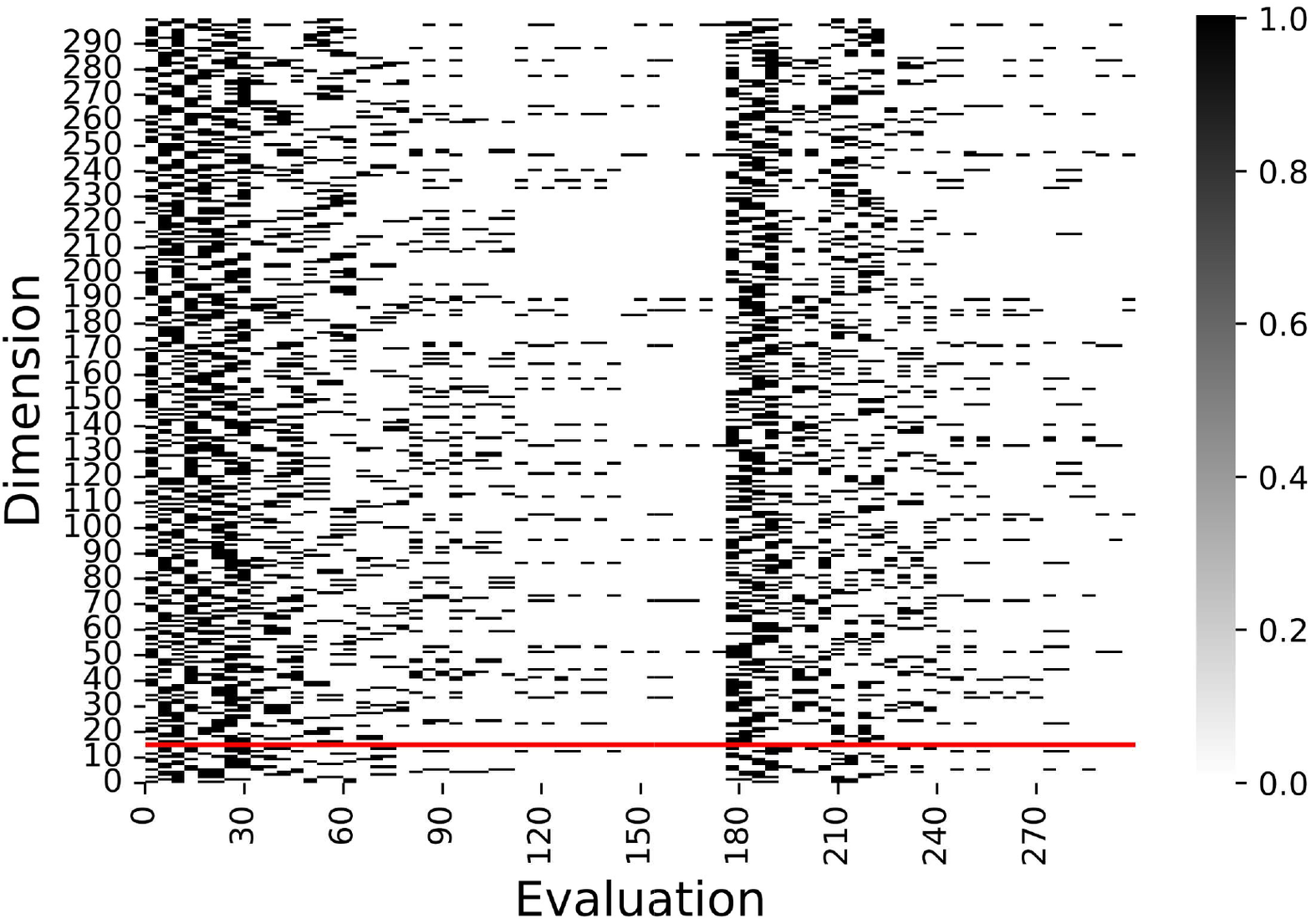}
    \label{fig:ss20300c}}
    \caption{We compare the variable selection between LassoBO and MCTS when performed on the Sum Squares function($d_e=15, D= 300$). \textbf{Left:} We compare the number of selected variables through function evaluations in LassoBO and MCTS\_VS. \textbf{Middle:} We depict the selected dimensions for each evaluation in LassoBO. \textbf{Right: }We depict the selected dimensions for each evaluation in MCTS\_VS.}
\label{fig:ss20300}
\end{figure*}
\paragraph{Experiments on extremely low and high-dimensional problems.}
We also evaluate the compared methods for extremely low and high dimensional problems by testing on Levy with $D=100,d_e = 15$ and $D=500,d_e=15$. As expected, the right subfigure of Fig. \ref{fig:levy15_100500} shows that LassoBO has a clear advantage over the rest methods on the extremely high-dimensional function $D=500$. The left subfigure shows that on $D=100$, LassoBO behaves the best and TuRBO is the runner-up, implying that LassoBO can also tackle low-dimensional
problems to some degree.

\begin{figure*}[h]
    \centering
    \includegraphics[width=0.325\textwidth]{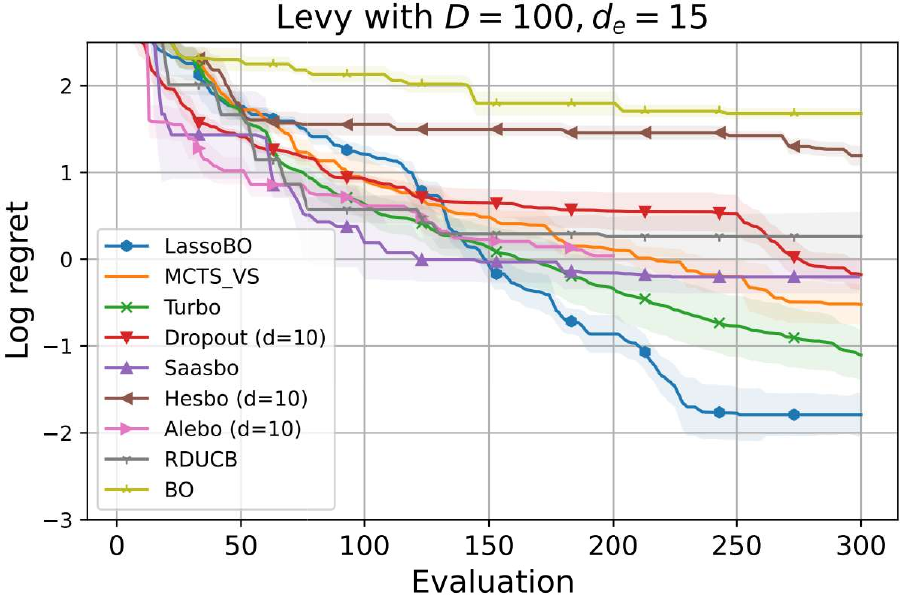}
    \includegraphics[width=0.325\textwidth]{levy15_300.pdf}
    \includegraphics[width=0.325\textwidth]{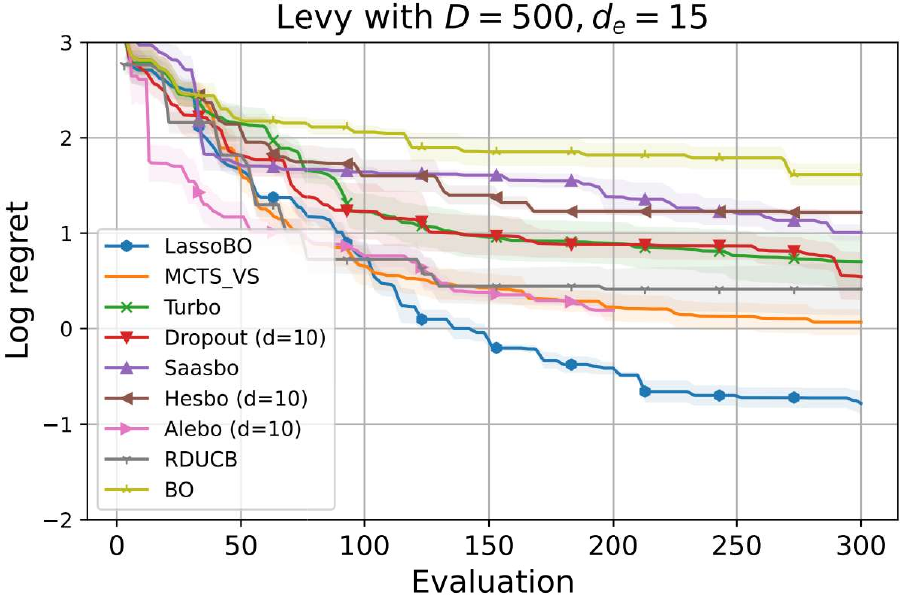}
    \caption{Performance comparison on extremely low and high dimensional problems.}
    \label{fig:levy15_100500}
\end{figure*}

\paragraph{Analysis on $M_t$:} We conducted additional ablation studies on $M_t$, as shown in the Figure \ref{fig:subspace}. In this experiment, we test LassoBO with different $M_t$ on the Levy function with $D=300, d_e=10$. As illustrated in the figure, when $M_t$ is small, the algorithm performs well in the early stages due to its high likelihood of exploiting regions already identified as promising. However, it struggles to achieve further improvement in later stages. Conversely, when $M_t$ is large, the algorithm explores a broader search space, resulting in poorer initial performance. Nevertheless, this broader exploration enhances estimation efficiency and leads to better convergence in the later stages.
\begin{figure*}[h!]
    \centering
    \includegraphics[width=0.45\textwidth]{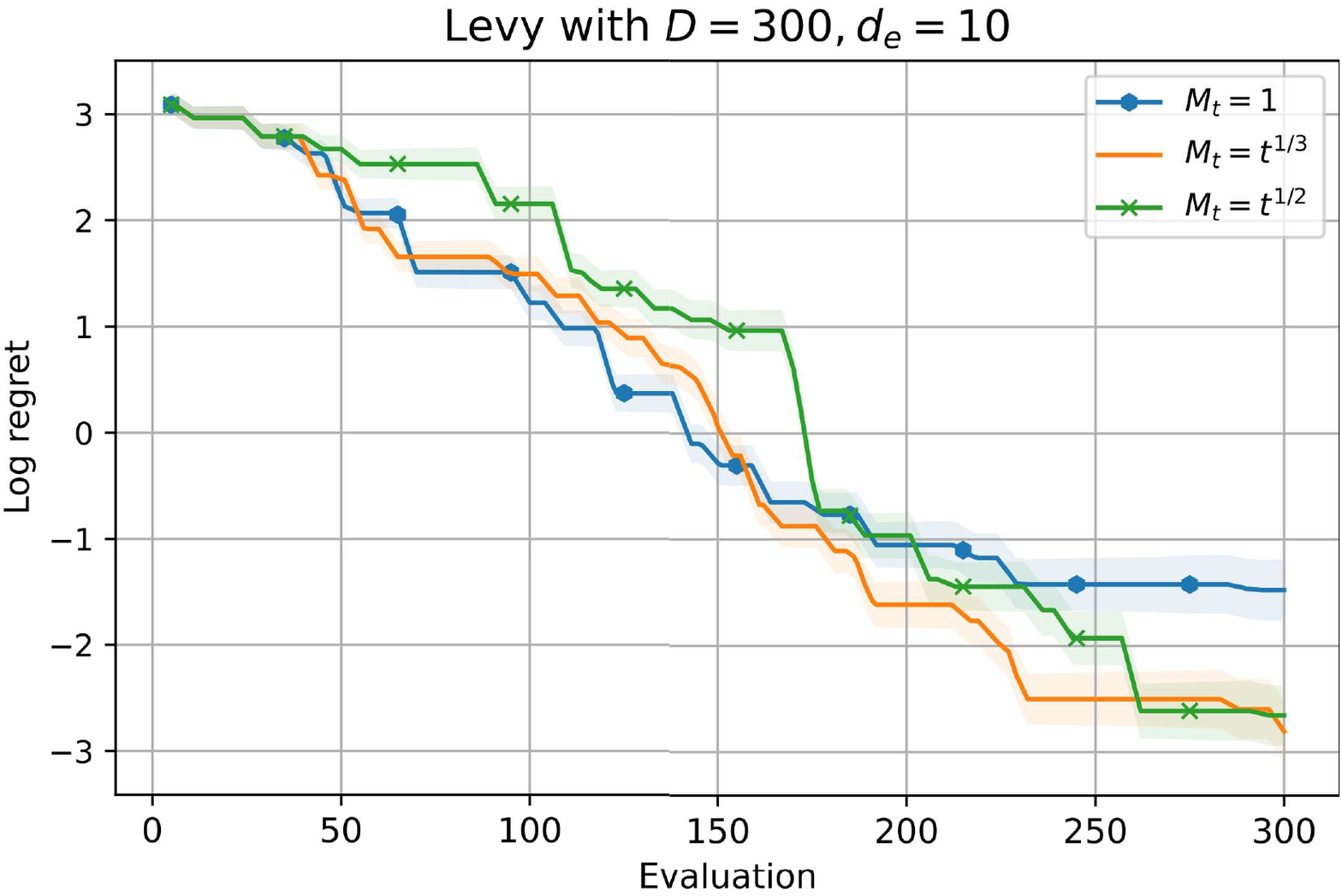}
    \caption{Comparison results using different subspaces $M_t$.}
    \label{fig:subspace}
\end{figure*}

\paragraph{Ablation study on the window size.}

\begin{figure*}
    \centering
    \includegraphics[width = 0.42 \textwidth ]{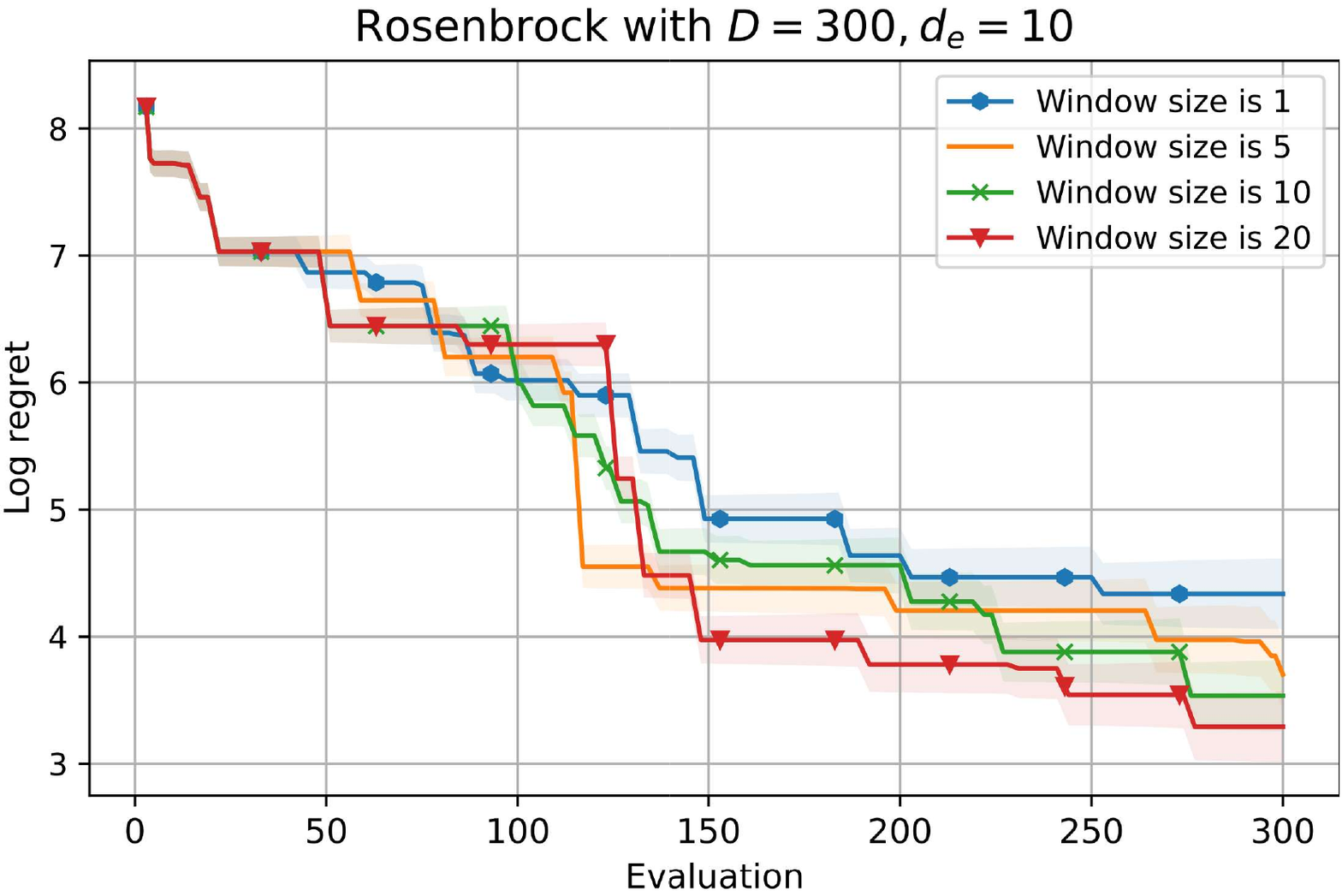}
    \includegraphics[width = 0.42 \textwidth ]{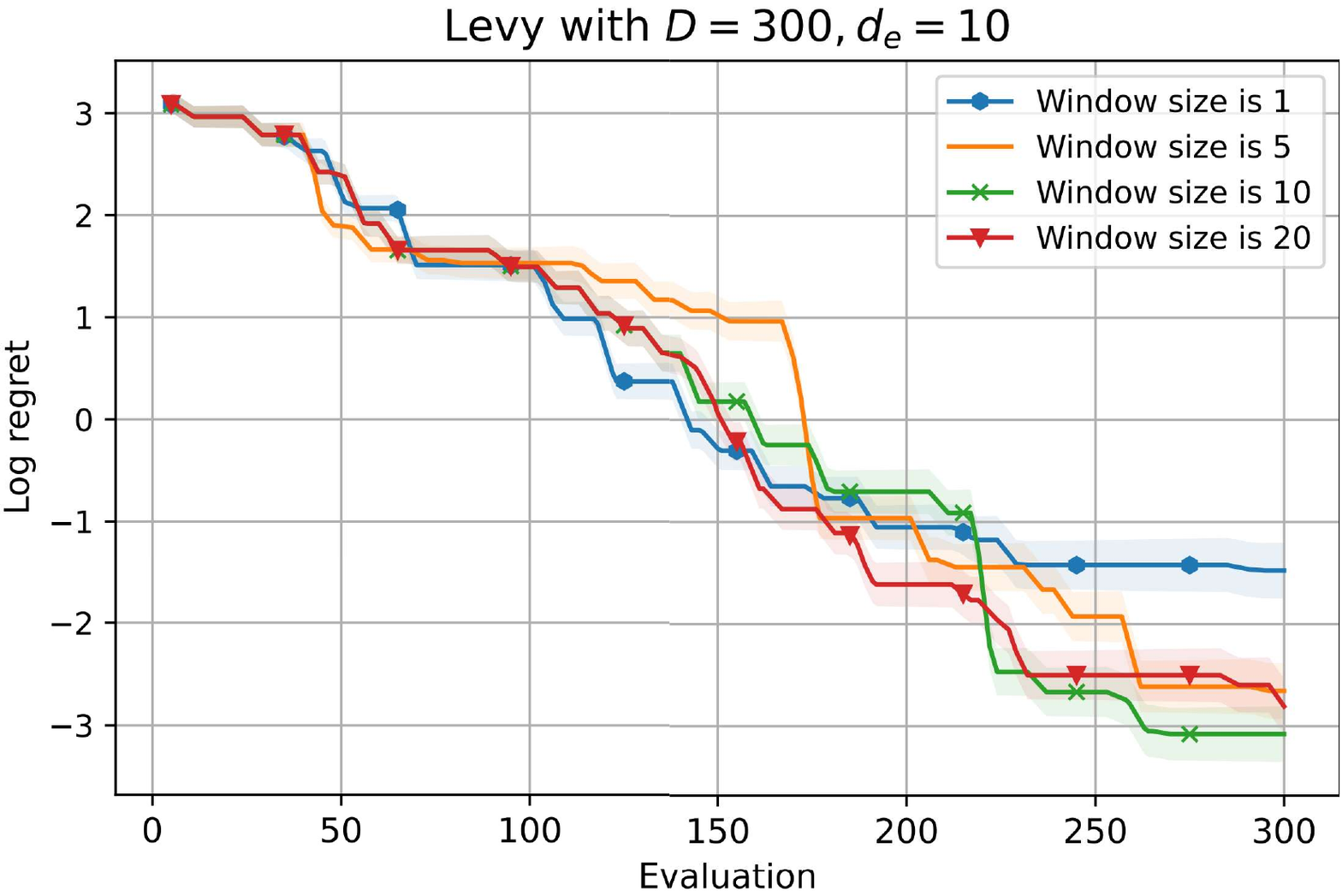}
    \caption{Comparison results using different window sizes for variable selection. The figure demonstrates that the effect of window size on the performance of LassoBO depends on different problems.}
    \label{fig:windowsize}
\end{figure*}
In this study, we consider using the median of window size $W$ of $\mathbf{\rho}$ (i.e. $ \mathbf{\rho}^(t), \mathbf{\rho}^{(t-1)} ,\dots,\mathbf{\rho}^{(t-W+1)}$) to select the variable in the $t$th loop instead of only using $\mathbf{\rho}_t$. To study the influence of the window size of LassoBO, we set $W \in \{1,5,10, 20\}.$ The results are shown in Fig. \ref{fig:windowsize}. We see that the effect of window size on the performance of LassoBO depends on different problems. However, with a window size of $10$, the performance of LassoBO is relatively good for both functions. Therefore,  we recommend using these window sizes.

\section{Additional details on the implementation and the empirical evaluation}
\label{sec:inforimplement}
\subsection{Baselines}
To maximize the regularized marginal log-likelihood in Eq. (\ref{eq:optimizing_Ut}), we sample 10 initial hyperparameters. The five best samples are further optimized using the ADAM optimizer for 100 steps. If not mentioned otherwise, we default to choosing \(\lambda = 10^{-3}\) and \(M_t = \sqrt[3]{t}\) for all experiments. To update the hyperparameters by maximizing $U_t$, we seek a gradient of the partial derivatives of $U_t$ w.r.t. the hyperparameters.

We benchmark against SAASBO, TuRBO, HeSBO, Alebo, and CMA-ES:
\begin{itemize}
    \item For SAASBO, we use the implementation from \cite{eriksson2021high} (\url{https://github.com/ martinjankowiak/SAASBO}, license: none, last accessed: 05/09/2022).
\item For TURBO, we use the implementation from \cite{eriksson2019scalable} (\url{https://github.com/uber-research/TuRBO}, license: Uber, last accessed: 05/09/2022).
\item For HeSBO and AlEBO, we use the implementation from \cite{letham2020re} (\url{https://github.com/ facebookresearch/alebo}, license: CC BY-NC 4.0, last accessed: 05/09/2022).
\item For RDUCB, we use the implementation from \cite{ziomek2023random} (\url{https://github.com/huawei-noah/HEBO/tree/master/RDUCB})
\end{itemize}
 The experiments are conducted on an AMD Ryzen 9 5900HX CPU @ 3.3GHz, and 16 GB of RAM, and we use a single thread.
 \subsection{Real world benchmark}

\paragraph{Rover trajectory optimization problem.} LassoBO is tested on the rover trajectory optimization problem presented in \cite{wang2018batched} where the task is to find an optimal trajectory through a 2d-environment. In the original problem, the trajectory is determined by fitting a B-spline to 30 waypoints and the goal is to optimize the locations of these waypoints.
\paragraph{MuJoCo.} Second, it is tested on the more difficult MuJoCo tasks \citep{todorov2012mujoco} in Reinforcement Learning. The goal is to find the parameters of a linear policy maximizing the accumulative reward. The objective f (i.e., the accumulative reward) is highly stochastic here, making it difficult to solve. We use the mean of ten independent evaluations to estimate $f$.
\paragraph{DNA.} The DNA benchmark \citep{nardi2022lassobench} is a biomedical classification task, taking binarized DNA quences as input. We use the negative loss function as the objective function for the evaluation.
\section{Limitation}
A limitation of our work is that the theoretical guarantees of LassoBO rely on a few assumptions. For example, the regularity assumption that assumes that the objective
function $f$ is a GP sample may not be true in some problems.  Moreover, LassoBO relies on the assumption of low effective dimensionality and might not work well if the percentage of valid variables is high.

\vfill


\end{document}